\documentclass[preprint]{elsarticle}

\usepackage{amsthm,amssymb,amsbsy,amsmath,amsfonts,amssymb,amscd}
\usepackage{graphicx}
\usepackage{dsfont}
\usepackage{ulem}
\usepackage{setspace}
\usepackage{subfigure}
\usepackage{cancel}
\usepackage{bbm}
\usepackage{multirow}
\usepackage[colorlinks=true,linkcolor=black,citecolor=black,urlcolor=black,plainpages=false]{hyperref}

\newcommand{\V}{\mathrm{var}}
\newcommand{\C}{\mathrm{cov}}
\newcommand\dx[1]{\mathrm{d} \mu(#1)}
\newcommand\dix[2]{\mathrm{d} \mu_{#2}(#1)}
\newcommand\PS[1]{{\langle #1 \rangle}_\mathcal{H}}

\newcommand\N[1]{{\left| \left| #1 \right| \right|}_\mathcal{H}}

\newtheorem{hyp}{H}
\newtheorem{propo}{Proposition}

\theoremstyle{definition}

\newenvironment{coro}[1][Corollary]{\begin{trivlist}
\item[\hskip \labelsep {\bfseries #1}]}{\end{trivlist}}

\title{ANOVA kernels and RKHS of zero mean functions for model-based sensitivity analysis}

\author[SHEFF,EMSE]{N.~Durrande\corref{cor1}}
\ead{durrande@emse.fr}
\author[BERNE]{D.~Ginsbourger}
\ead{ginsbourger@gmail.com}
\author[EMSE]{O.~Roustant}
\ead{roustant@emse.fr}
\author[TELEC]{L.~Carraro}
\ead{laurent.carraro@telecom-st-etienne.fr}

\cortext[cor1]{Corresponding author. Nicolas Durrande, email: durrande@emse.fr, phone: (+44) (0)114 222 3824.}

\address[SHEFF]{School of mathematics and statistics, University of Sheffield, Sheffield S3 7RH, UK}
\address[EMSE]{Ecole Nationale Sup\'erieure des Mines, FAYOL-EMSE, LSTI, F-42023 Saint-Etienne, France}
\address[BERNE]{Institute of Mathematical Statistics and Actuarial Science, University of Berne, Alpeneggstrasse 22 - 3012 Bern, Switzerland}
\address[TELEC]{TELECOM St Etienne, 25 rue du Dr R\'emy Annino - 42000 Saint Etienne, France}

\begin{document}


\begin{abstract}
Given a reproducing kernel Hilbert space $(\mathcal{H}, \langle ., . \rangle )$ of real-valued functions and a suitable measure $\mu$ over the source space $D \subset \mathds{R}$, we decompose $\mathcal{H}$ as the sum of a subspace of centered functions for $\mu$ and its orthogonal in $\mathcal{H}$. 
This decomposition leads to a special case of ANOVA kernels, for which the functional ANOVA representation of the best predictor can be elegantly derived, either in an interpolation or regularization framework. The proposed kernels appear to be particularly convenient for analyzing the effect of each (group of) variable(s) and computing sensitivity indices without recursivity. 
\end{abstract}

\maketitle

\section{Introduction}

Let $f$ be a real-valued function defined over $D \subset \mathds{R}^d$. 
We assume that $f$ is costly to evaluate and that we want to study some global properties of $f$ such as the influence of a variable $x_i$ on the output $f(x_1,\dots,x_d)$. 
As the number of evaluations of $f$ is limited, it may be unaffordable to run sensitivity analysis methods directly on $f$. 
Thus, it can be helpful to replace $f$ by a mathematical approximation for performing such studies~\cite{Marrel2009}. 
In this article we propose a class of models that is well suited for performing global sensitivity analysis. 
First, we present some background in sensitivity analysis, interpolation in Reproducing Kernel Hilbert Space (RKHS), and on the class of ANOVA kernels. 
Then we construct RKHS of zero mean functions and derive a subclass of ANOVA kernels that is well suited for the ANOVA representation of the best predictor and for global sensitivity analysis. 
The use of this new class of kernels is illustrated throughout the article on classical test functions from the sensitivity analysis literature. 

\subsection{Sensitivity analysis}

The purpose of global sensitivity analysis is to describe the inner structure of $f$ and to analyse the influence of each variable or group of variables on $f$~\cite{saltelli2008global}. A traditional approach is to study the variance of $f(\mathbf{X})$ where $\mathbf{X}$ is a random vector over $D$. Hereafter, we assume that $f \in L^2(D,\mu)$, where $D=D_1 \times  \dots \times D_d$ is a Cartesian product space of sets $D_i \subset \mathds{R}$ and where $\mu=\mu_1 \otimes \dots \otimes \mu_d$ is a product of probability measures over $D_i$. The measure $\mu$ describes the variability of the inputs and we define $\mathbf{X}$ as a random vector with probability distribution $\mu$. Note that although this framework is often considered in sensitivity analysis, the product structure of $\mu$ implies that the components of $\mathbf{X}$ are independent.

\medskip

For $d=1$, any $g \in L^2(D,\mu)$ can be decomposed as a sum of a constant plus a zero mean function~\cite{gu2002smoothing},
\begin{equation*}
g =  \int_D g(s) \dx{s}  + \left(g - \int_D g(s) \dx{s}  \right).
\end{equation*}
\noindent
The two elements of this decomposition are orthogonal for the usual $L^2$ scalar product so we have a geometric decomposition of $L^2(D,\mu)$:
\begin{equation}
\begin{split}
L^2(D,\mu) & = L^2_1(D,\mu) \stackrel{\perp}{\oplus} L^2_0(D,\mu)
\end{split}
\label{eq:decL21}
\end{equation}
\noindent
where $L^2_1(D,\mu)$ denotes the subspace of constant functions and $L^2_0(D,\mu)$ the subspace of zero mean functions: $L^2_0(D,\mu) = \{g \in L^2(D,\mu) : \int_D g(s) \dx{s} =0 \}$.

\medskip

Similarly, if $d>1$, the space $L^2(D,\mu)$ has a tensor product structure~\cite{Kree74} 
\begin{equation}
L^2(D,\mu) = \bigotimes_{i=1}^d L^2(D_i,\mu_i). 
\label{eq:L2dec}
\end{equation}
Using Eq.~\ref{eq:decL21} and the notation $L^2_{B}(D,\mu)= \bigotimes_{i=1}^d L^2_{B_i}(D_i,\mu_i)$ for $B \in \{ 0,1\}^d$ we obtain
\begin{equation}
L^2(D,\mu)  = \bigotimes_{i=1}^d \left ( L^2_1(D_i,\mu_i) \stackrel{\perp}{\oplus} L^2_0(D_i,\mu_i) \right ) = \bigoplus_{B \in \{ 0,1\}^d}^\perp L^2_{B}(D,\mu).
\label{eq:L2dectot}
\end{equation}
A key property is that two subspaces $L^2_{B}$ and $L^2_{B'}$ are orthogonal whenever $B \neq B'$.
Given an arbitrary function $f \in L^2(D,\mu)$, the orthogonal projection of $f$ onto those subspaces leads to the functional ANOVA representation~\cite{Efron1981,Sobol2001} (or \textit{Hoeffding-Sobol decomposition}) of $f$ into main effects and interactions:
\begin{equation}
f(\mathbf{x}) = f_0 + \sum_{i=1}^d f_i(x_i) + \sum_{i < j} f_{i,j}(x_i,x_j) + \dots + f_{1,\dots,d}(\mathbf{x}),
\label{eq:decf}
\end{equation}
\noindent
where $f_0$ is the orthogonal projection of $f$ onto the space of constant functions $L^2_{\{0 \}^d}(D,\mu)$ and where $f_{I}$ ($I \subset \{1, \ldots, d\}$) is the orthogonal projection of $f$ onto $L^2_{\mathrm{ind}(I)}(D,\mu)$ with $\mathrm{ind}(I)_i = 1$ if $i \in I$ and $\mathrm{ind}(I)_i=0$ if $i \notin I$. By construction, the integral of $f_{I}$ with respect to any of the variables indexed by $i\in I$ is zero. This decomposition gives an insight into the influence of each set of variables $\mathbf{x}_I=\{x_i, i \in I \}$ on $f$. For the constant term, the main effects, and the two-factor interactions, one gets the classical expressions~\cite{gu2002smoothing}:
\begin{equation}
\begin{split}
f_0 & = \int_D f(\mathbf{x}) \dx{\mathbf{x}} \\
f_i(x_i) & = \int_{D_{-i}} f(\mathbf{x}) \dix{\mathbf{x}_{-i}}{-i} - f_0 \\
f_{i,j}(x_i,x_j) & = \int_{D_{-\{i,j\}}} f(\mathbf{x}) \dix{\mathbf{x}_{-\{i,j\}}}{-\{i,j\}} - f_i(x_i) - f_j(x_j) - f_0
\end{split} 
\label{eq:expdecf}
\end{equation}
with the notations $D_{-I}=\prod_{i\notin I}D_{i}$ and $\mu_{-I}=\bigotimes_{i \notin I}\mu_{i}$. As one can see from Eq.~\ref{eq:expdecf}, the calculation of any $f_I$ requires having recursively computed all the $f_{J}$'s for $J \subset I$, which makes it cumbersome (if not practically impossible) to get high order interactions.  

\medskip

Recall the random vector $\mathbf{X}$ with distribution $\mu$. $L^2(D,\mu)$-orthogonality between any two terms of the decomposition implies that the variance of the random variable $f(\mathbf{X})$ can be decomposed as
\begin{equation}
\V(f(\mathbf{X})) = \sum_{i=1}^d \V(f_i(X_i)) + \sum_{i<j} \V(f_{i,j}(\mathbf{X}_{i,j})) + \dots + \V(f_{1,\dots,d}(\mathbf{X})).
\label{eq:decvar1}
\end{equation} 
For the subsets $I \subset \{1, \ldots, d\}$, the global sensitivity indices $S_I$ (also called Sobol indices) are then defined as
\begin{equation}
S_I= \frac{\V(f_I(\mathbf{X}_I))}{\V(f(\mathbf{X}))}.
\label{eq:Sind}
\end{equation}
$S_I$ represents the proportion of variance of $f(\mathbf{X})$ explained by the interaction between the variables indexed by $I$. The knowledge of the indices $S_I$ is very helpful for understanding the influence of the inputs, but the computation of the $f_I$'s is cumbersome when evaluation of $f$ is costly since they rely on the computation of the integrals of Eq.~\ref{eq:expdecf}. Following \cite{Marrel2009}, it can then be advantageous to perform the sensitivity analysis on a surrogate model $m$ approximating $f$. 

\subsection{Optimal approximation in RKHS}

\noindent
The class of functional approximation techniques considered in this work, commonly referred to as Kriging or Gaussian process regression in contemporary statistical learning settings, can be seen as methods of optimal interpolation in RKHS. $f$ is here assumed to be known at a set of points $\mathcal{X} = \{\mathcal{X}_1, \dots, \mathcal{X}_n\}$ with $ \mathcal{X}_i \in D$. Given a RKHS $\mathcal{H}$ of real-valued functions over $D$ with reproducing kernel $K(.,.)$, we want to select a function of $\mathcal{H}$ that is a good approximation of $f$. Two options are usually considered~\cite{Rasmussen2006}: the first one is to solve an interpolation problem, and the second a regularization problem. 

\medskip

The interpolation approach consists of finding a function $h \in \mathcal{H}$ that satisfies $h(\mathcal{X}_i)=f(\mathcal{X}_i)$ for $i=1,\dots,n$ (for which we will use the vectorial notation $h(\mathcal{X})=f(\mathcal{X})$). As there might be an infinite number of function of $\mathcal{H}$ that satisfy this criterion, the optimal interpolator is defined as~\cite{Rasmussen2006}:
\begin{equation}
m(\mathbf{x}) = \underset{h \in \mathcal{H}}{\mathrm{argmin}} \left( \N{h} : h(\mathcal{X})= f(\mathcal{X}) \right). 
\label{eq:krikri}
\end{equation}
When the Gram matrix $\mathrm{K}$ of general term $\mathrm{K}_{i,j} = K(\mathcal{X}_i,\mathcal{X}_j)$ is invertible, we obtain
\begin{equation}
m(\mathbf{x}) = \mathbf{k}(\mathbf{x})^t \mathrm{K}^{-1} \mathbf{F} 
\label{eq:krikrisol}
\end{equation}
where $\mathbf{F} = f(\mathcal{X})$ is the column vector of observations and $\mathbf{k}(.)$ is the column vector of functions $\left(K(\mathcal{X}_i,.)\right)_{1 \leq i \leq n}$.

\medskip

For the regularization approach, the best predictor $\tilde{m}$ is defined to be the minimizer of:
\begin{equation}
\sum_{i=1}^n (h(\mathcal{X}_i) - F_i)^2 + \lambda \N{h}^2,
\label{eq:krikrireg}
\end{equation}
where the parameter $\lambda $ is a positive number tuning the trade-off between the norm of $\tilde{m}$ and the distance to the observations. The solution of this minimization problem is
\begin{equation}
\tilde{m}(\mathbf{x}) = \mathbf{k}(\mathbf{x})^t (\mathrm{K}+ \lambda \mathrm{I})^{-1} \mathbf{F}.
\label{eq:krikriregsol}
\end{equation}
From the probabilistic point of view, the regularization corresponds to approximating $f$ based on noisy observations : $f_{obs}(\mathcal{X}_i) = f(\mathcal{X}_i) + \varepsilon_i$ where the $\varepsilon_i$ are independent random variables with distribution $\mathcal{N}(0,\lambda)$~\cite{Rasmussen2006}. 

\medskip

As it appears in Eqs~\ref{eq:krikrisol} and~\ref{eq:krikriregsol}, $m$ and $\tilde{m}$ are a linear combinations of the $K(\mathcal{X}_i,.)$. These basis functions and the influence of the parameter $\lambda$ are illustrated in Figure~\ref{fig:toy} on a toy example for two popular kernels: 
\begin{equation}
k_b(x,y)=\mathrm{min}(x,y) \text{\qquad and \qquad} k_g(x,y)=\mathrm{exp}\left(-(x-y)^2\right),
\label{eq:noytest}
\end{equation}
known respectively as the \textit{Brownian} and the \textit{Gaussian} covariance kernels. 
\begin{figure}[ht]
\centering
\subfigure[$k_b(\mathcal{X}_i,.)$]{
\includegraphics[width=4cm]{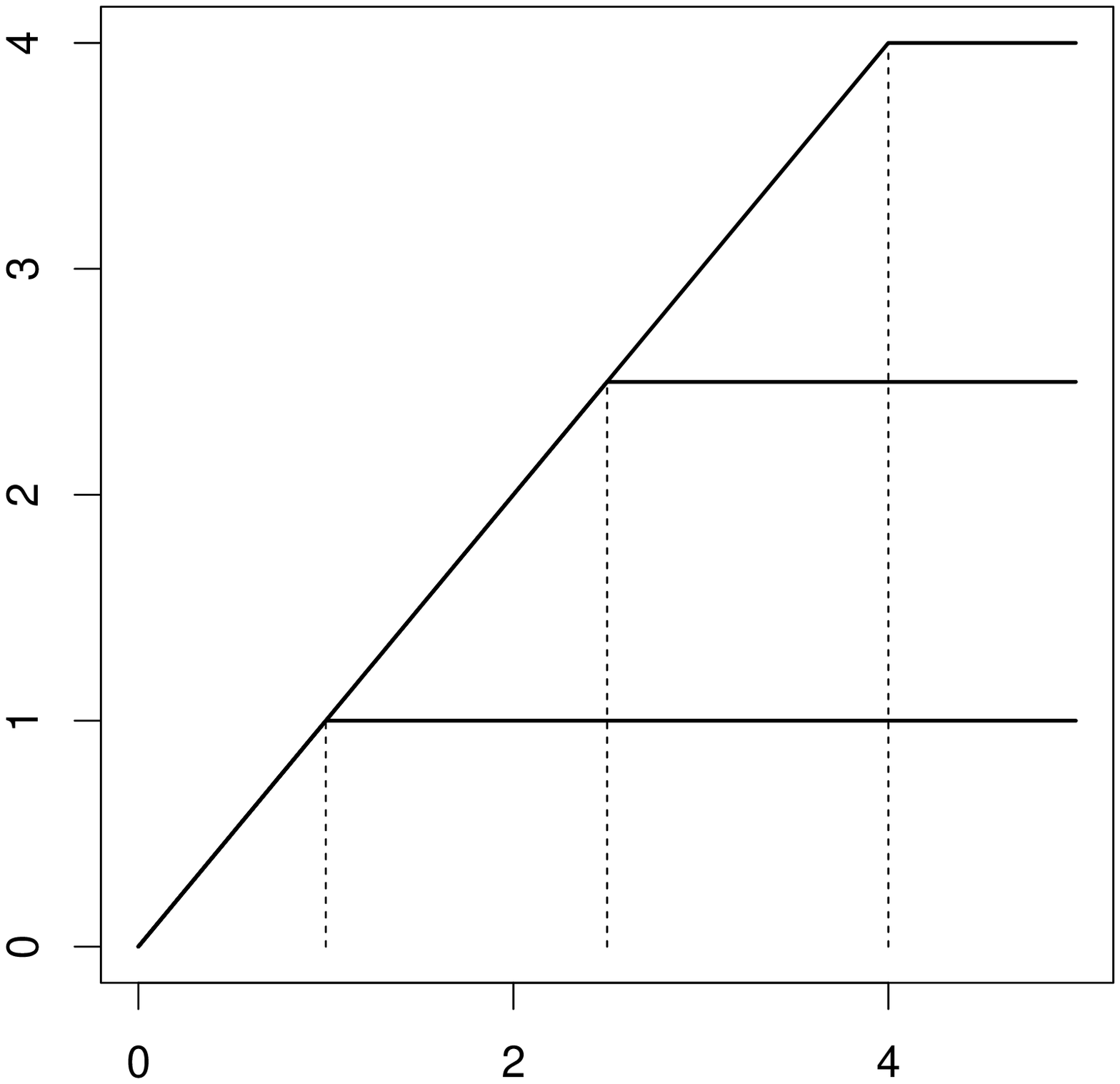}
\label{fig:s1fig1}
}
\subfigure[models based on $k_b$]{
\includegraphics[width=4cm]{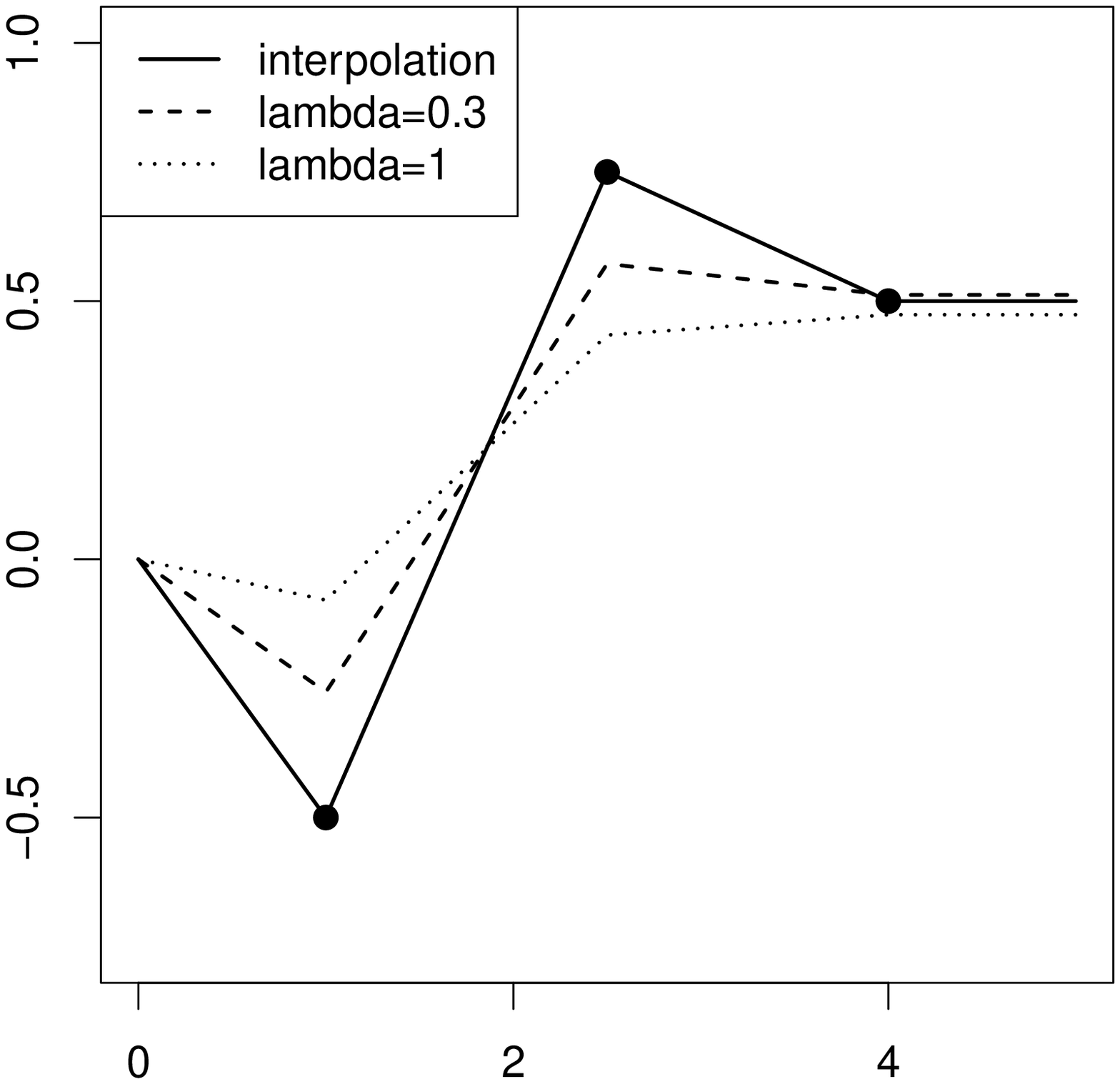}
\label{fig:s1fig2}
}
\subfigure[$k_g(\mathcal{X}_i,.)$]{
\includegraphics[width=4cm]{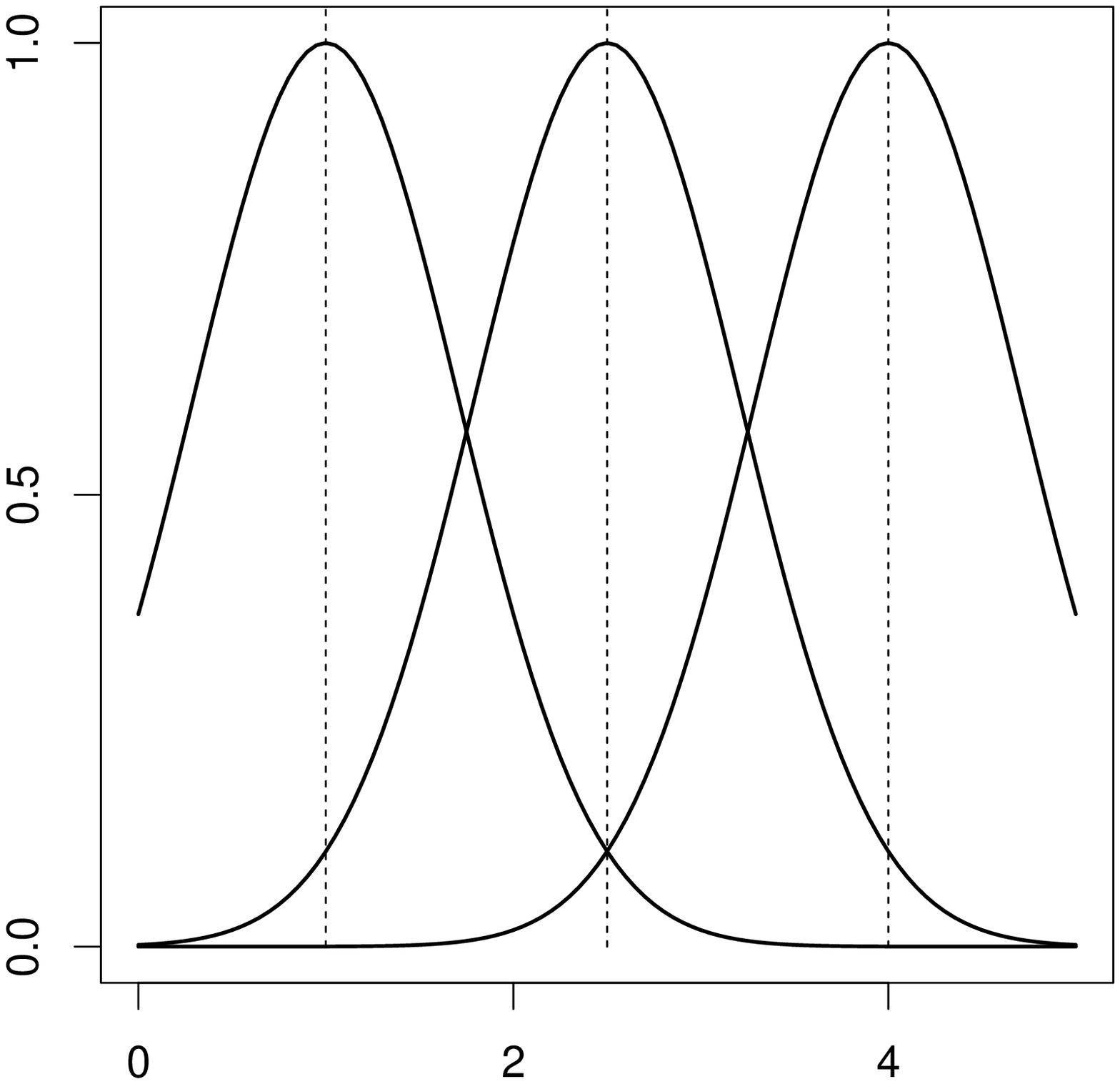}
\label{fig:s1fig3}
}
\subfigure[models based on $k_g$]{
\includegraphics[width=4cm]{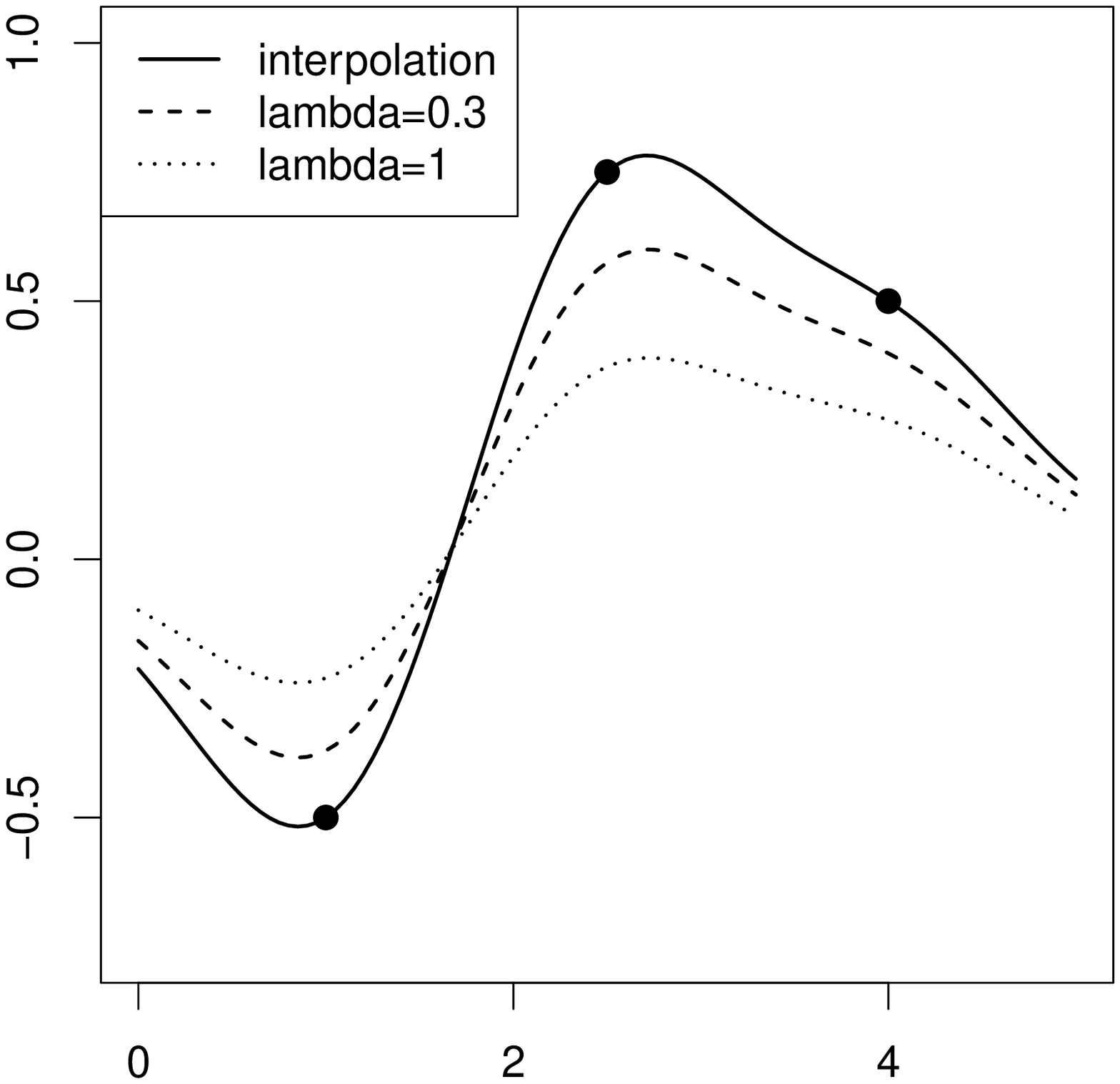}
\label{fig:s1fig4}
}
\caption{Examples of models obtained with kernels $k_b$ and $k_g$ with $D=[0,5]$. The set $\mathcal{X}=\{1,2.5,4 \}$ and the observations $\mathbf{F}=(-0.5,0.75,0.5)^t$ are chosen arbitrarily. The left panels represent the basis functions on which the model is based and the right panels show the obtained models $m$ and $\tilde{m}$ for two values of $\lambda$.}
\label{fig:toy}
\end{figure}

\medskip

The interpolation problem can be seen as a particular case of a regularization problem where $\lambda \rightarrow 0$~\cite{gu2002smoothing}. For the sake of clarity, we focus hereafter on the best interpolator $m$. However the general case of the regularization framework is addressed in Section~\ref{sec:reg}. 

\medskip

The kernel $K$ can be any symmetric positive definite (s.p.d.)\ function~\cite{vangeepuram2010reproducing} and it has to be chosen in practice. It is customary to arbitrarily select $K$ from usual s.p.d.\ functions (e.g.\ Mat\'ern, or power-exponential kernel families) according to some prior knowledge about $f$ such as its regularity. However, this choice has great impact on the resulting model (e.g.\ Figure~\ref{fig:toy}). If a kernel depends on some parameters, a traditional approach is to estimate them by maximum likelihood to ensure the kernel is adapted to the data~\cite{Rasmussen2006}.

\medskip

In the next section we focus on a particular family of s.p.d.\ functions called ANOVA kernels, which are designed to offer a separation of the terms with different orders of interaction. The main contribution of this paper, which is presented in the third section, deals with a special case of ANOVA kernels tailored for an improved disentanglement of multivariate effects. 

\subsection{ANOVA kernels and a candidate ANOVA-like decomposition of the best predictor}

The name ``ANOVA kernel'' has been introduced in the literature of machine learning by Stitson et Al.~\cite{Stitson1997} in the late 90s. These kernels allow control of the order of interaction in a model and enhance their interpretability~\cite{Gunn2002, Berlinet2004}. For $D=D_1 \times \dots \times D_d$, they are constructed as a product of univariate kernels $1+k^i$, where $1$ stands for a \textit{bias term} and the $k^i$'s are arbitrary s.p.d.\ kernels on $D_i \times D_i$ ($1\leq i \leq d$):
\begin{equation}
K_{ANOVA}(\mathbf{x},\mathbf{y}) = \prod_{i=1}^d (1 + k^i(x_i,y_i)) = 1 + \sum_{I \subset \{1,\dots,d \}} \prod_{i \in I} k^i(x_i,y_i).
\label{eq:KKANOVA}
\end{equation}
As shown on the right hand side of Eq.~\ref{eq:KKANOVA}, ANOVA kernels can also be seen as a sum of separable kernels with increasing interaction orders~\cite{duvenaud2011additive11}. 

\medskip

Denoting by $\mathds{1}^i$ and $\mathcal{H}^i$ the RKHS of functions defined over $D_i$ with respective reproducing kernels 1 and $k^i$, 
$K_{ANOVA}$ is the reproducing kernel of the space~\cite{Aronszajn1950,Berlinet2004}
\begin{equation}
\mathcal{H}_{ANOVA} = \bigotimes_{i=1}^d (\mathds{1}^i + \mathcal{H}^i) = \mathds{1} + \sum_{I \subset \{1,\dots,d \}} \mathcal{H}_I.
\label{eq:new}
\end{equation}
with $ \mathcal{H}_I = \bigotimes_{i \in I} \mathcal{H}^i \otimes \bigotimes_{i \notin I} \mathds{1}^i $.
\medskip

From Eq.~\ref{eq:KKANOVA} the particular structure of $K_{ANOVA}$ allows us to develop the $n \times 1$ vector $\mathbf{k}(\mathbf{x})$ of 
Eq.~\ref{eq:krikrisol} as follows:
\begin{equation}
\mathbf{k}(\mathbf{x}) = \mathbf{1} + \sum_{I \subset \{1,\dots,d \}} \bigodot_{i \in I} \mathbf{k}^i(x_i)
\label{eq:deck}
\end{equation}
where $\odot$ denotes a element-wise product: $ \left( \bigodot_{i \in I} \mathbf{k}^i(x_i) \right)_j = \prod_{i \in I}k^i(x_i,x_j)$. From this relation, we can get the decomposition of the best predictor $m$ given in Eq.~\ref{eq:krikrisol} onto the subspaces $\mathcal{H}_I$:
\begin{equation}
m(\mathbf{x})  
 = \mathbf{1}^t \mathrm{K}^{-1} \mathbf{F} + \sum_{I \subset \{1,\dots,d \}} \left ( \bigodot_{i \in I} \mathbf{k}^i(x_i) \right ) ^t \mathrm{K}^{-1} \mathbf{F}
\label{eq:mdec}
\end{equation}

Noting $m_0=\mathbf{1}^t \mathrm{K}^{-1} \mathbf{F}$ and $m_I(\mathbf{x})=\left ( \bigodot_{i \in I} \mathbf{k}^i(x_i) \right ) ^t \mathrm{K}^{-1} \mathbf{F}$,
we obtain an expression for $m$ which is similar to its ANOVA representation:
\begin{equation}
m(\mathbf{x}) = m_0 + \sum_{i=1}^d m_i(x_i) + \sum_{i<j} m_{i,j}(\mathbf{x}_{i,j}) + \dots + m_{1,\dots,d}(\mathbf{x}_{1,\dots,d}).
\label{eq:decm}
\end{equation}

In this expression, the $m_I $ have the useful feature of not requiring any recursive computation of integrals. However, this decomposition differs from the ANOVA representation of $m$ since the properties of Eq.~\ref{eq:decf} are not respected. Indeed, the $m_I$ of Eq.~\ref{eq:decm} are not necessarily zero mean functions and any two terms of the decomposition are generally not orthogonal in $L^2(D,\mu)$. This point is illustrated numerically on a test function at the end of the second example. It can be seen theoretically if the $k_i$ are Ornstein-Uhlenbeck kernels. In this case, it is known that $\mathds{1}^i \subset \mathcal{H}^i$~\cite{Antoniadis1984}. As all the $\mathcal{H}_I$ contain $\mathds{1}$ they are obviously not orthogonal.

\medskip

An alternative to avoid this issue is to consider RKHS such that for all $i$, $\mathcal{H}^i$ is $L^2$-orthogonal to the space of constant functions $\mathds{1}^i$ (i.e.\ the $\mathcal{H}^i$ are spaces of zero mean functions for $\mu_i$). This construction ensures that the decomposition of Eq.~\ref{eq:mdec} has the properties required in Eq.~\ref{eq:decf} so we benefit from the advantages of the two equations: the meaning of Eq.~\ref{eq:decf} for the analysis of variance; and the simplicity of computation for the $m_I$'s from Eq.~\ref{eq:mdec}. This approach has been applied to splines under the name of Smoothing Spline ANOVA (SS-ANOVA)~\cite{Wahba1995,gu2002smoothing}. However, SS-ANOVA is based on a decomposition of the inner product of a spline RKHS and it cannot be immediately applied for other RKHS. In the next section we introduce a general method for extracting RKHS of zero mean functions from any usual RKHS. As this approach is not limited to splines, it can be seen as an extension of the SS-ANOVA 
framework. 

\section{RKHS of zero mean functions}

\subsection{Decomposition of one-dimensional RKHS}
Let $\mathcal{H}$ be a RKHS of functions defined over a set $D \subset \mathds{R}$ with kernel $k$ and $\mu$ a finite Borel measure over $D$. Furthermore, we consider the hypothesis:
\begin{hyp}
\label{hyp1D1}
\begin{tabular}{ll}
 (i) & $k:\ D \times D \rightarrow \mathds{R}$ is $\mu \otimes \mu$-measurable. \\
 (ii) & $\displaystyle \int_D \sqrt{k(s,s)} \dx{s} < \infty$.
\end{tabular}
\end{hyp}
Note that when is $\mu$ a probability measure, any bounded kernel (e.g.\ Gaussian, power-exponential, Mat\'ern) satisfies the condition $(ii)$. Thus H\ref{hyp1D1} is not a very restrictive hypothesis. 

\begin{propo}
\label{prop1}
Under H\ref{hyp1D1}, $\mathcal{H}$ can be decomposed as a sum of two orthogonal sub-RKHS, $\mathcal{H}= \mathcal{H}_0 \stackrel{\perp}{\oplus} \mathcal{H}_1$ where $\mathcal{H}_0$ is a RKHS of zero-mean functions for $\mu$, and its orthogonal $\mathcal{H}_1$ is at most 1-dimensional.
\end{propo}
\begin{proof}
From H\ref{hyp1D1}, the integral operator $\displaystyle I : \mathcal{H} \rightarrow \mathds{R},\ h  \mapsto \int_D h(s) \dx{s}$
is bounded since
\begin{equation}
|I(h)|  \leq \int_D |\PS{h,k(s,.)}| \dx{s} \leq \N{h} \int_D \sqrt{k(s,s)} \dx{s} 
\label{eq:Icont}
\end{equation}
According to the Riesz representation theorem, there exists a unique $R \in \mathcal{H}$ such that $\forall h \in \mathcal{H}$, $I(h)= \PS{h,R}$. 
If $R(.) = 0$, then all $f \in \mathcal{H}$ are centered functions for $\mu$, so that $\mathcal{H}_{0}=\mathcal{H}$ and $\mathcal{H}_{1}=\{0\}$.
If $R(.) \neq 0$, then $\mathcal{H}_1=span(R)$ is a 1-dimensional sub-RKHS of $\mathcal{H}$, and the subspace $\mathcal{H}_0$ of centered functions 
for $\mu$ is characterized by $\mathcal{H}_0=\mathcal{H}_1^\perp$. 
\end{proof}

For all $x \in D$ the value of $R(x)$ can be calculated explicitly. Indeed, recalling that $k(x,.)$ and $R$ are respectively the representers in $\mathcal{H}$ 
of the evaluation functional at $x$ and of the integral operator, we get:
\begin{equation}
R(x) = \PS{k(x,.),R} = I(k(x,.)) = \int_D k(x,s) \dx{s}.
\label{eq:RIx}
\end{equation}

The reproducing kernels $k_0$, $k_1$ of $\mathcal{H}_0$ and $\mathcal{H}_1$ satisfy $k=k_0 + k_1$. 
Let $\pi$ denote the orthogonal projection onto $\mathcal{H}_1$. Following~\cite[theorem 11]{Berlinet2004} we obtain
\begin{equation}
\begin{split}
k_0(x,y) &= k(x,y) - \pi(k(x,.))(y) \\
&= k(x,y) - \frac{\displaystyle \int_D k(x,s) \dx{s}  \int_D k(y,s) \dx{s}}{\displaystyle \iint_{D \times D} k(s,t) \dx{s} \dx{t}} \\
\end{split}
\label{eq:k1}
\end{equation}

\paragraph{Example 1}
Let us illustrate the previous results for the kernels $k_b$ and $k_g$ introduced in Eq.~\ref{eq:noytest}. As previously, we choose $D=[0,5]$ and define $\mu$ as the uniform measure over $D$. Following Proposition 1, $k_b$ and $k_g$ can be decomposed as a sum of two reproducing kernels:
\begin{equation}
\begin{split}
k_b(x,y) & = k_{b,0}(x,y) + k_{b,1}(x,y) \\
k_g(x,y) & = k_{g,0}(x,y) + k_{g,1}(x,y).
\end{split}
\end{equation}
Figure~\ref{fig:EX1} represents sections of $k_{b,i}(x,y)$ and $k_{g,i}(x,y)$ for various values of $y$. We observe from this figure that $k_{b,0}(.,y)$ and $k_{g,0}(.,y)$ are zero mean functions. Moreover, $k_{b,0}(.,y)$ and $k_{b,1}(.,y)$ (respectively $k_{g,0}(.,y)$, $k_{g,1}(.,y)$) are orthogonal for the inner product induced by $k_b$ (resp. $k_g$) but they are not orthogonal for $L^2(D,\mu)$.
\begin{figure}[ht]
\centering
\subfigure[$k_{b,i}(.,0)$]{
\includegraphics[width=3.5cm]{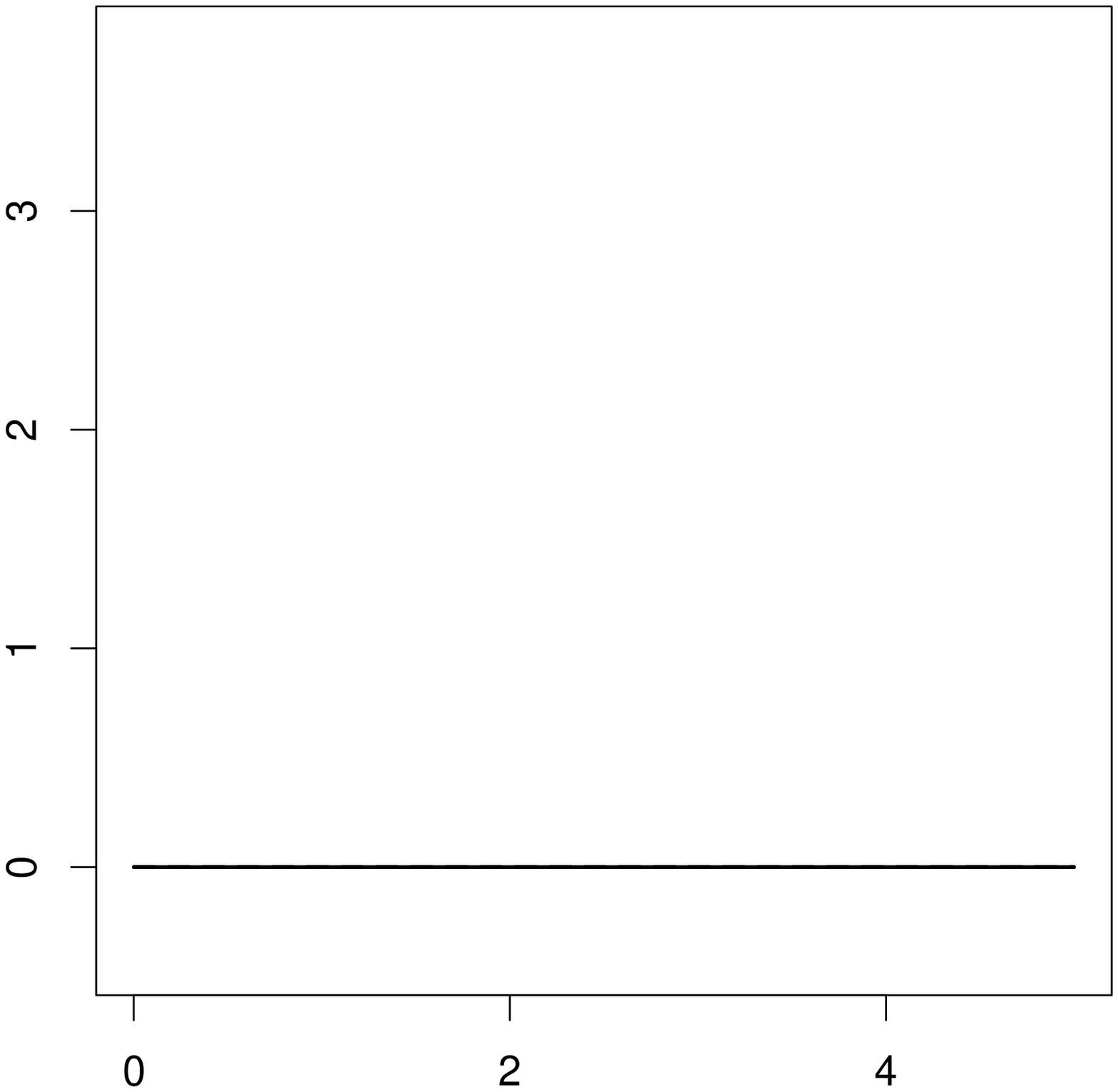}
\label{fig:sfig1}
}
\subfigure[$k_{b,i}(.,2)$]{
\includegraphics[width=3.5cm]{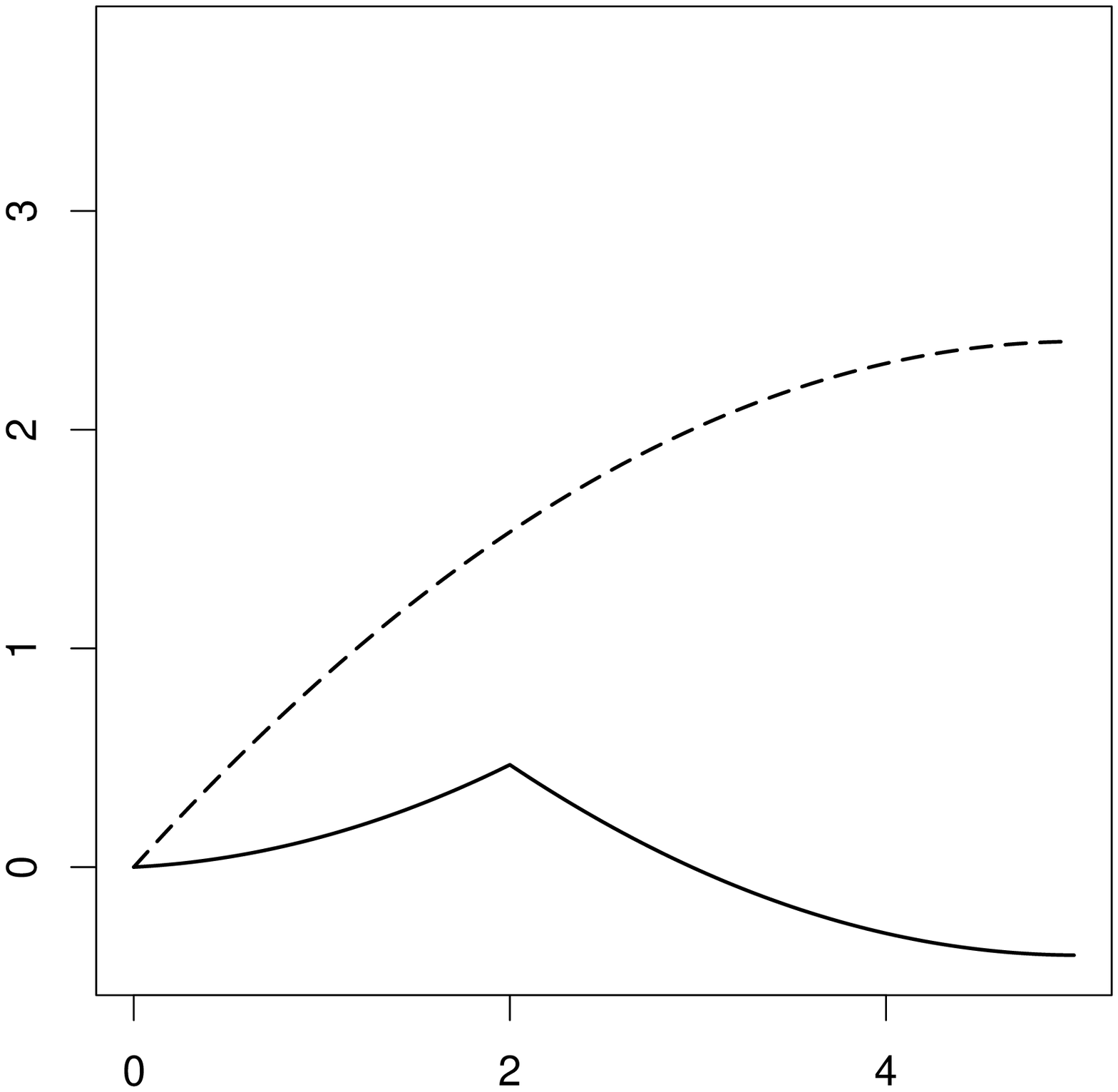}
\label{fig:sfig2}
}
\subfigure[$k_{b,i}(.,4)$]{
\includegraphics[width=3.5cm]{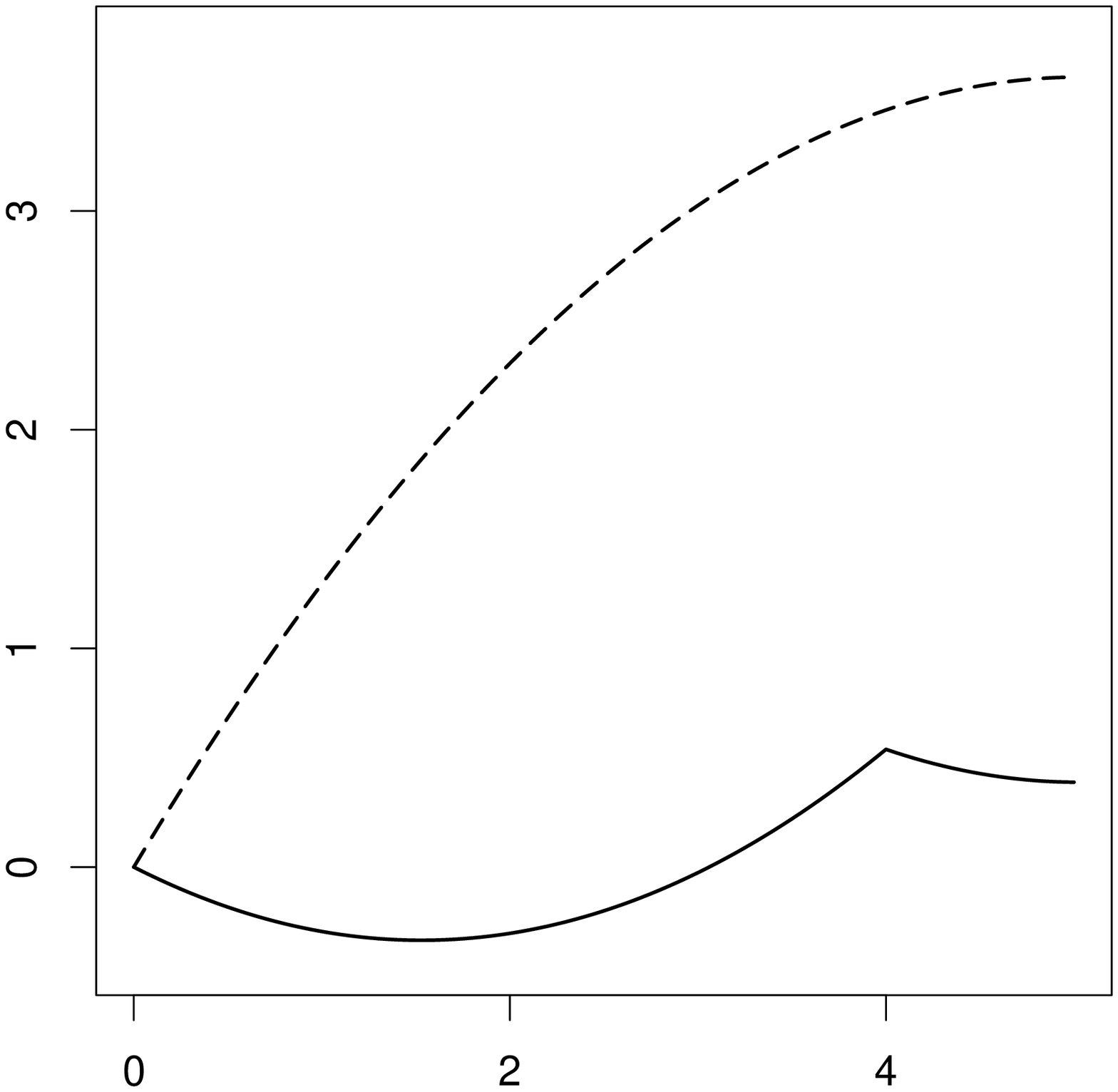}
\label{fig:sfig3}
}
\subfigure[$k_{g,i}(.,0)$]{
\includegraphics[width=3.5cm]{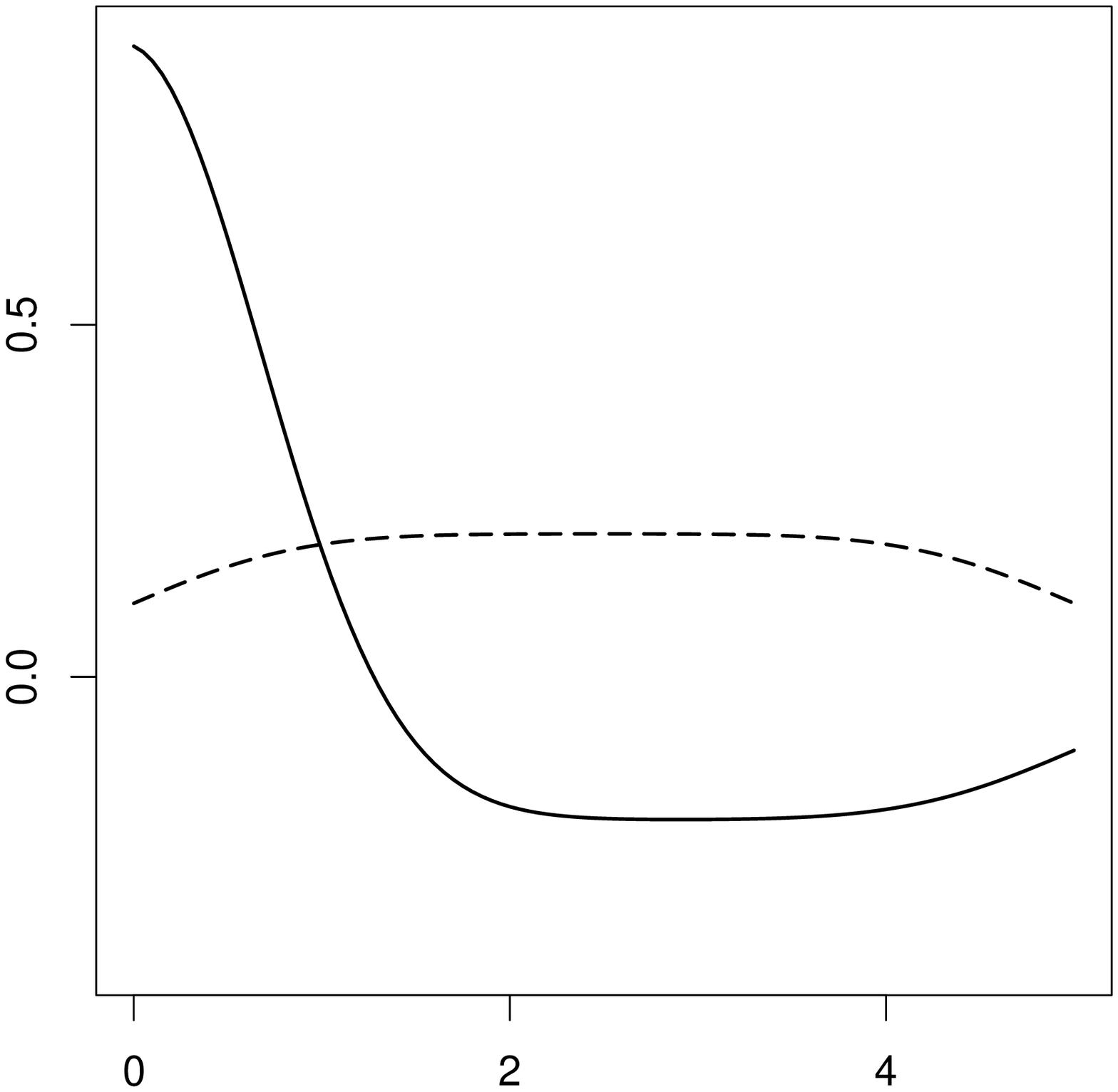}
\label{fig:sfig4}
}
\subfigure[$k_{g,i}(.,2)$]{
\includegraphics[width=3.5cm]{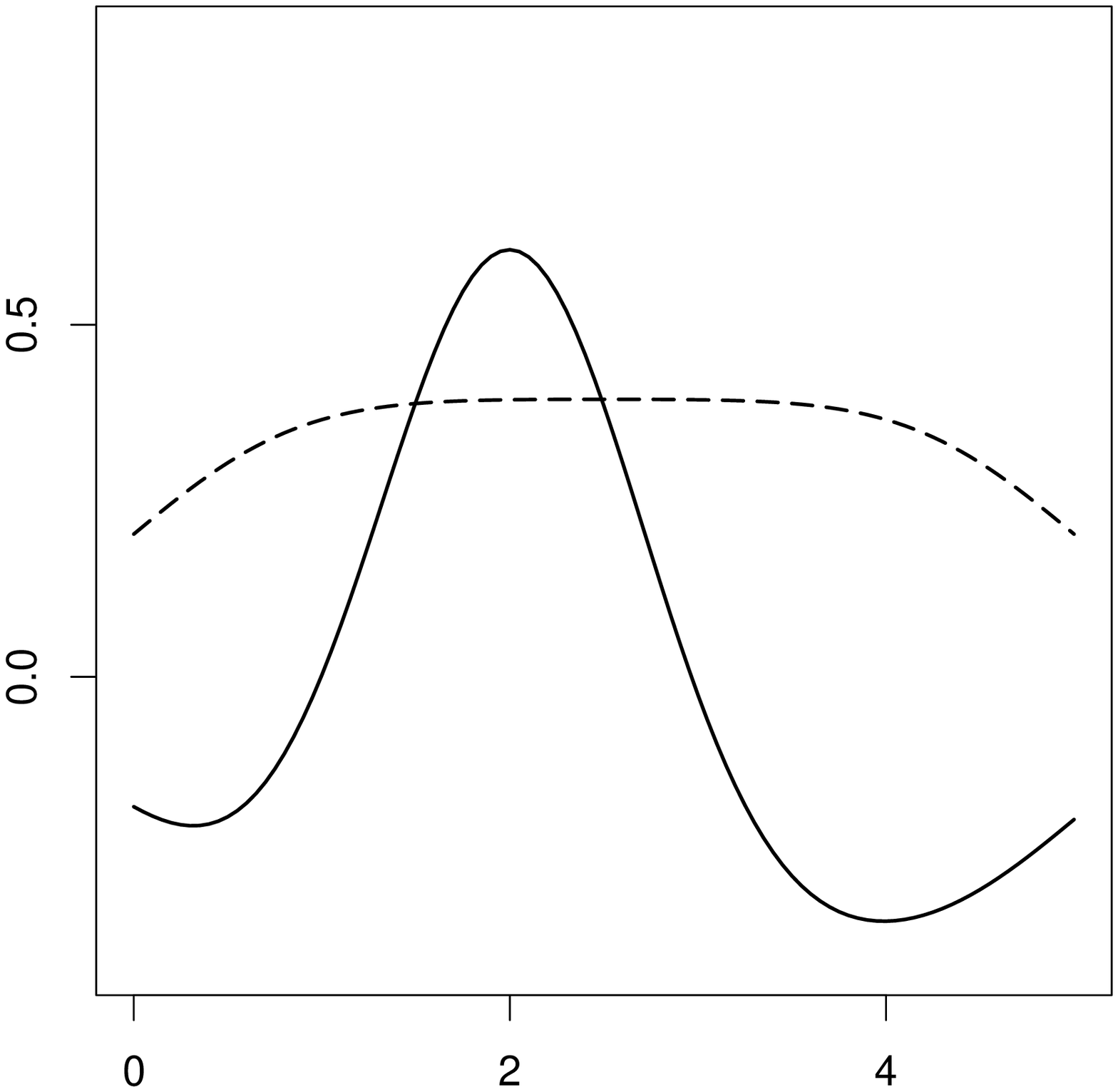}
\label{fig:sfig5}
}
\subfigure[$k_{g,i}(.,4)$]{
\includegraphics[width=3.5cm]{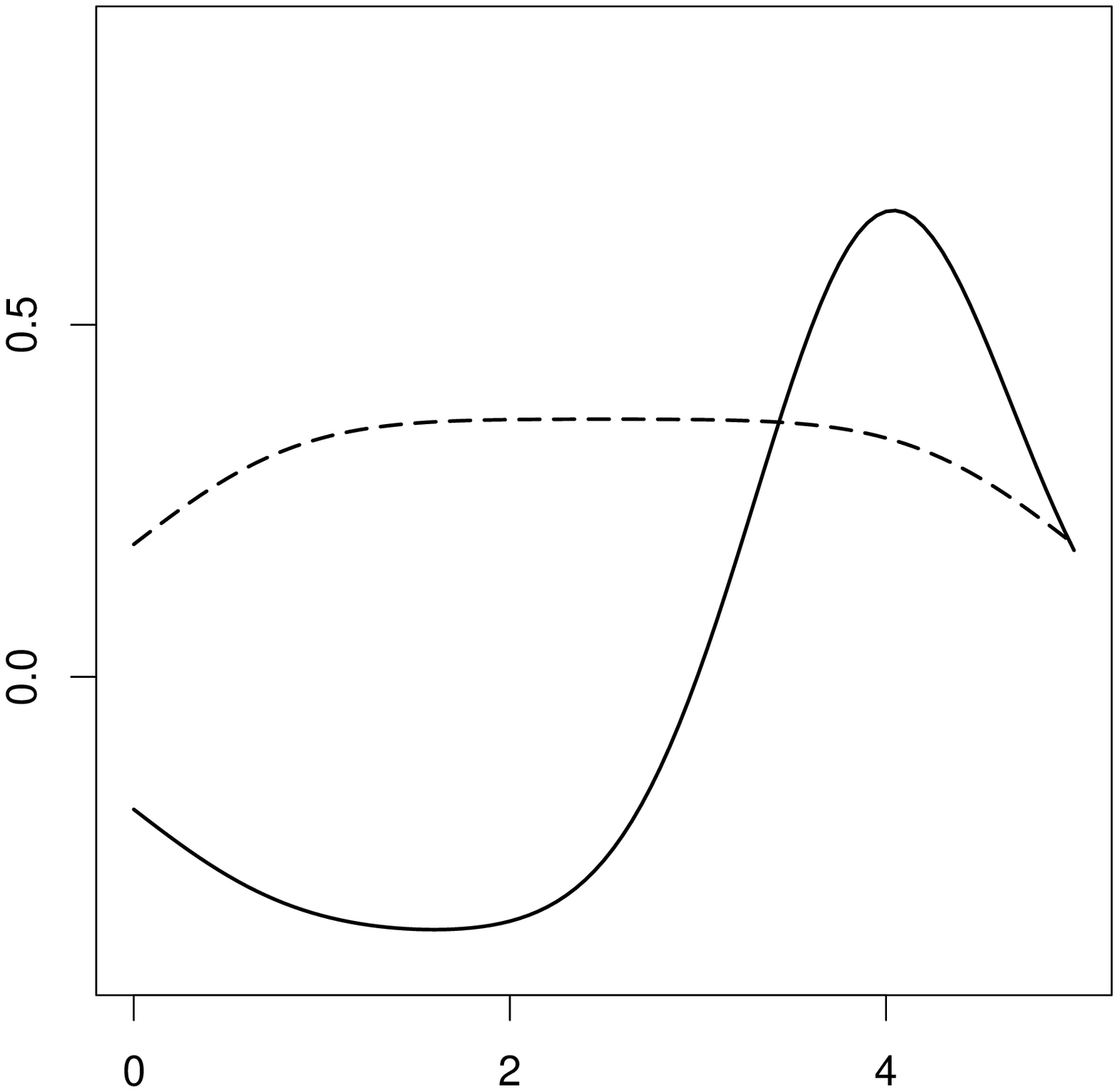}
\label{fig:sfig6}}
\caption[Optional caption for list of figures]{Representation of the kernels $k_{b,i}(.,y)$ and $k_{g,i}(.,y)$ for $y=0,2,4$ and $i=0,1$. The dashed lines correspond to $k_{b,1}$, $k_{g,1}$ and the solid lines are for $k_{b,0}$ and $k_{g,0}$. The original kernels $k_b(.,y)$ and $k_g(.,y)$ can be retrieved by summing the two curves.}
\label{fig:EX1}
\end{figure}

\subsection{Remarks on the numerical computation of \texorpdfstring{$k_0$}{k0} and \texorpdfstring{$k_1$}{k1}}
In the previous example, the integrals of Eq.~\ref{eq:k1} have been approximated using a Riemann sum operator $I_R$, and we discuss here the influence of this approximation. Let $\hat{k}_0$ be the approximation of $k_0$ using Riemann sums in Eq.~\ref{eq:k1}. By definition of a RKHS, the evaluation is linear continuous on $\mathcal{H}$ so the operator $I_R$ is also linear continuous and the pattern of the previous proof is still valid when $I$ is replaced by $I_R$. The function $\hat{k}_0$ corresponds in fact to Proposition~\ref{prop1} applied to the averaging operator $I_R$. We thus have the following result:

\begin{coro}[Corollary of Proposition 1:]
The function $\hat{k}_0$ is the reproducing kernel of the subspace of $\mathcal{H}$ of zero mean functions for $I_R$.
\end{coro}

This ensures that the approximation of Eq.~\ref{eq:k1} by Riemann sums is still a s.p.d.\ function. For all $x \in [0,5]$ we thus have $I_R(\hat{k}_0(x,.))=0$ and $I(k_0(x,.))=0$, but $I_R(k_0(x,.))$ and $I(\hat{k}_0(x,.))$ are generally not equal to zero. 

\medskip

Note that this remark about the Riemann sum is not valid for Monte Carlo approximations of integrals. Indeed, Monte Carlo integration has random output and the use of such methods in Eq.~\ref{eq:k1} leads to a random function that cannot be interpreted as a kernel (for example the obtained function is not symmetric with probability one).

\medskip

The kernel $k_b$ can be used to compare the error induced by replacing $I$ by $I_R$ since the integrals of Eq.~\ref{eq:k1} can be computed analytically. For $I_R$ based on a set of 100 points equally spaced over $[0, 5]$, the maximum distance we observed between $k_{b,0}$ and $\hat{k}_{b,0}$ is lower than $0.02$ so the two functions would be indistinguishable in Figure~\ref{fig:EX1}. Hereafter, the integrals will always be computed using Riemann sum approximations and we will not differentiate $k_0$ and $\hat{k}_0$ any more.

\subsection{Generalization for multi-dimensional RKHS}
The former decomposition of one-dimensional kernels leads directly to the decomposition of tensor product kernels if the $k^i$'s satisfy H\ref{hyp1D1}:
\begin{equation}
K(\mathbf{x},\mathbf{y}) = \prod_{i=1}^d k^i(x_i,y_i) = \prod_{i=1}^d (k^i_0(x_i,y_i) + k^i_1(x_i,y_i)).
\label{eq:KAD}
\end{equation}
Since $k^i_0(x_i,y_i)$ is a kernel of a one-dimensional RKHS, its effect is similar to the constant function $1$ in an ANOVA kernel. This decomposition highlights the similarity between usual tensor product kernels (power-exponential, Brownian, Mat\'ern...) and ANOVA kernels. 

\section{A new class of kernels for sensitivity analysis}

We now propose a special class of ANOVA kernels:
\begin{equation}
K_{ANOVA}^*(\mathbf{x},\mathbf{y}) = \prod_{i=1}^d (1 + k^i_0(x_i,y_i)) = 1 + \sum_{I \subset \{1,\dots,d \}} \prod_{i \in I} k^i_0(x_i,y_i),
\label{eq:KANOVA}
\end{equation}
where the $k^i_0$ are obtained by decomposing kernels as in the previous section. We now develop the use of such kernels for the ANOVA representation of the best predictor and for the computation of global sensitivity indices.

\subsection{ANOVA representation of \texorpdfstring{$m$}{m}}
\begin{propo}
\label{prop2}
If $m$ is a best predictor based on $K_{ANOVA}^*$, 
\begin{equation}
m_I = \left( \bigodot_{i \in I} \mathbf{k}^i_0(x_i) \right)^t \mathrm{K}^{-1} \mathbf{F}
\label{eq:propo2}
\end{equation}
is the term of the functional ANOVA representation of $m$ indexed by $I$. Hence, the decomposition of $m$ given by Eq.~\ref{eq:decm} coincides with its functional ANOVA representation (Eq.~\ref{eq:decf}). 
\end{propo}
\begin{proof}
The kernels $k_0^i$ are associated with RKHS $\mathcal{H}_0^i$ of zero-mean functions, so we have $\mathds{1}^i \perp_{L^2} \mathcal{H}^i_0$. The underlying RKHS associated with $K_{ANOVA}^*$ is
\begin{equation}
\mathcal{H}_{ANOVA}^* = \prod_{i=1}^d  \left ( \mathds{1}^i \stackrel{\perp}{\oplus} \mathcal{H}^i_0 \right )
\label{eq:Hstar}
\end{equation}
where $\perp$ stands for the $L^2$ inner product. The result follows.
\end{proof}

Unlike usual ANOVA kernels, the class of $K_{ANOVA}^*$ ensures that the terms $m_I$ are mutually orthogonal in the $L^2$ sense.

\paragraph{Example 2: ANOVA representation}
In order to illustrate the use of the $K_{ANOVA}^*$ kernels, we consider an analytical test function called the $g$-function of Sobol \cite{Saltelli2000}. This function is defined over $[0,1]^d$ by
\begin{equation}
g(x_1,\dots,x_d)= \prod_{k=1}^d \frac{|4x_k-2|+a_k}{1+a_k} \text{\qquad with } a_k > 0.
\label{eq:Gsob}
\end{equation}
We consider here the case $d=2$ and we set $a_1=1,\ a_2=2$. Starting from a one-dimensional Mat\'ern 3/2 kernel 
\begin{equation}
k_m(x,y) = (1+2 \, |x-y|) \exp(-2 \, |x-y|),
\label{eq:MK}
\end{equation}
we derive the expression of $K_{ANOVA}^*$ using Eq.~\ref{eq:k1} and~\ref{eq:KANOVA}:
\begin{equation}
K_{ANOVA}^*(\mathbf{x},\mathbf{y}) =  \prod_{i=1}^2 \left ( 1 + k_m(x_i,y_i) - \frac{\displaystyle \int_0^1 k_m(x_i,s) \dx{s}  \int_0^1 k_m(y_i,s) \dx{s}}{\displaystyle \iint_0^1 k_m(s,t) \dx{s} \dx{t}} \right ).
\label{eq:KASM}
\end{equation}

We then build the optimal interpolator $m \in \mathcal{H}_{ANOVA}^*$ based on the observation of $g$ at 20 points of $[0,1]^2$ (which result from a LHS-maximin procedure~\cite{Santner2003}). According to Proposition 2, the function $m$ can be split as a sum of 4 submodels $m_0$, $m_1$, $m_2$ and $m_{12}$ which are represented in Figure~\ref{fig:2}. We observe numerically that the mean value of $m_1$, $m_2$ and $m_{12}$ is lower than $10^{-14}$ (in absolute value), corroborating that these functions are zero-mean. Similarly, the value of the numerical computation of the scalar products between any two different functions of the set $\{m_0,m_1,m_2,m_{12} \}$ has the same order of magnitude.

\medskip
As a comparison, we consider on this test function the usual ANOVA kernel $K_{ANOVA}(\mathbf{x},\mathbf{y}) =  (1 + k_m(x_1,y_1)) \times (1 + k_m(x_2,y_2))$. For the corresponding model, the mean value of $m_1$, $m_2$ and $m_{12}$ is typically in the range $[-0.1,0.3]$ and a scalar product between $m_I$ and $m_J$ in $[-0.01,0.03]$ for $I \neq J$. 

\begin{figure}[!ht]
\centering
\subfigure[$g$]{
\includegraphics[width=3.5cm]{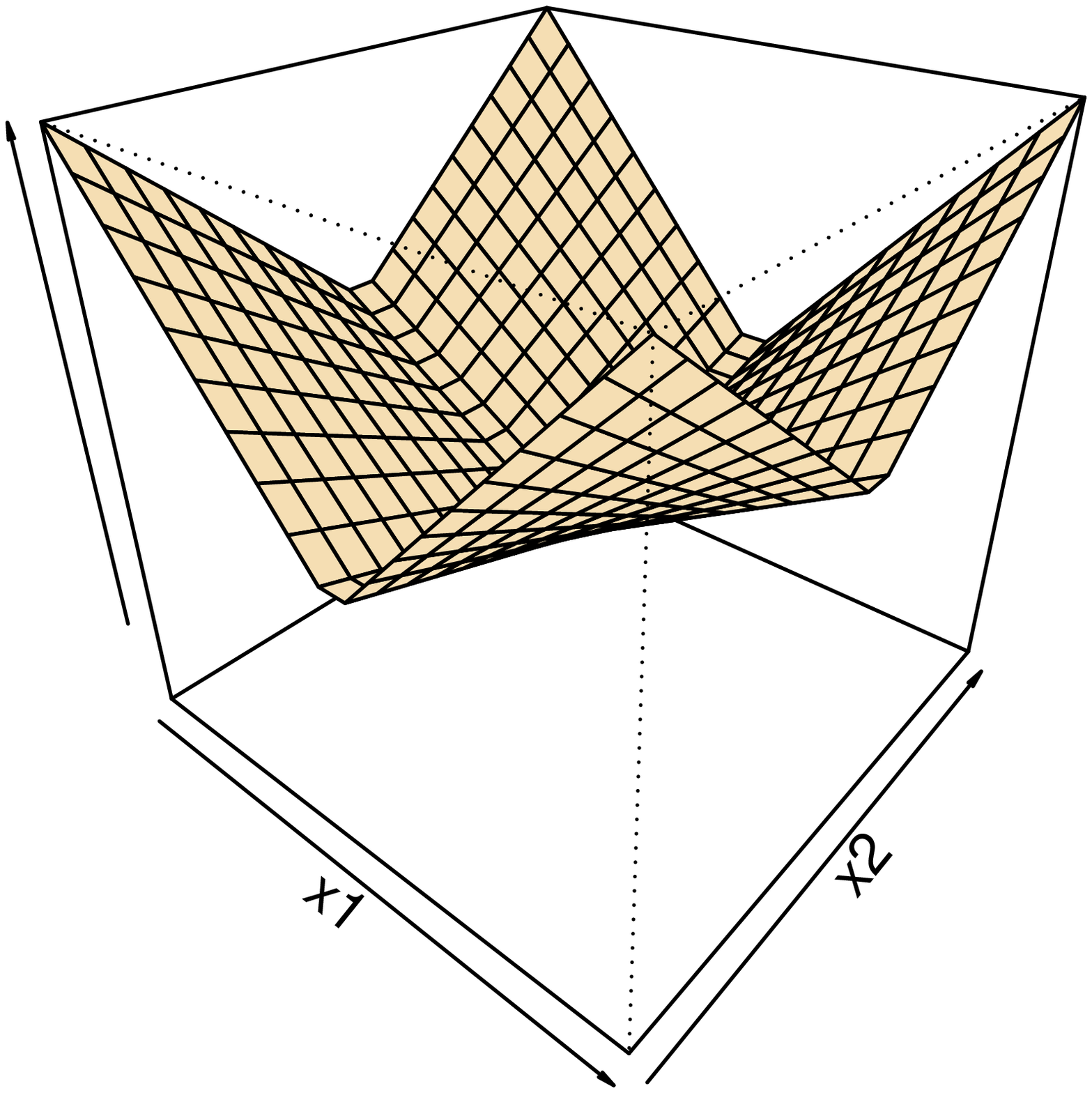}
\label{fig:subfig6}
}
\subfigure[$m$]{
\includegraphics[width=3.5cm]{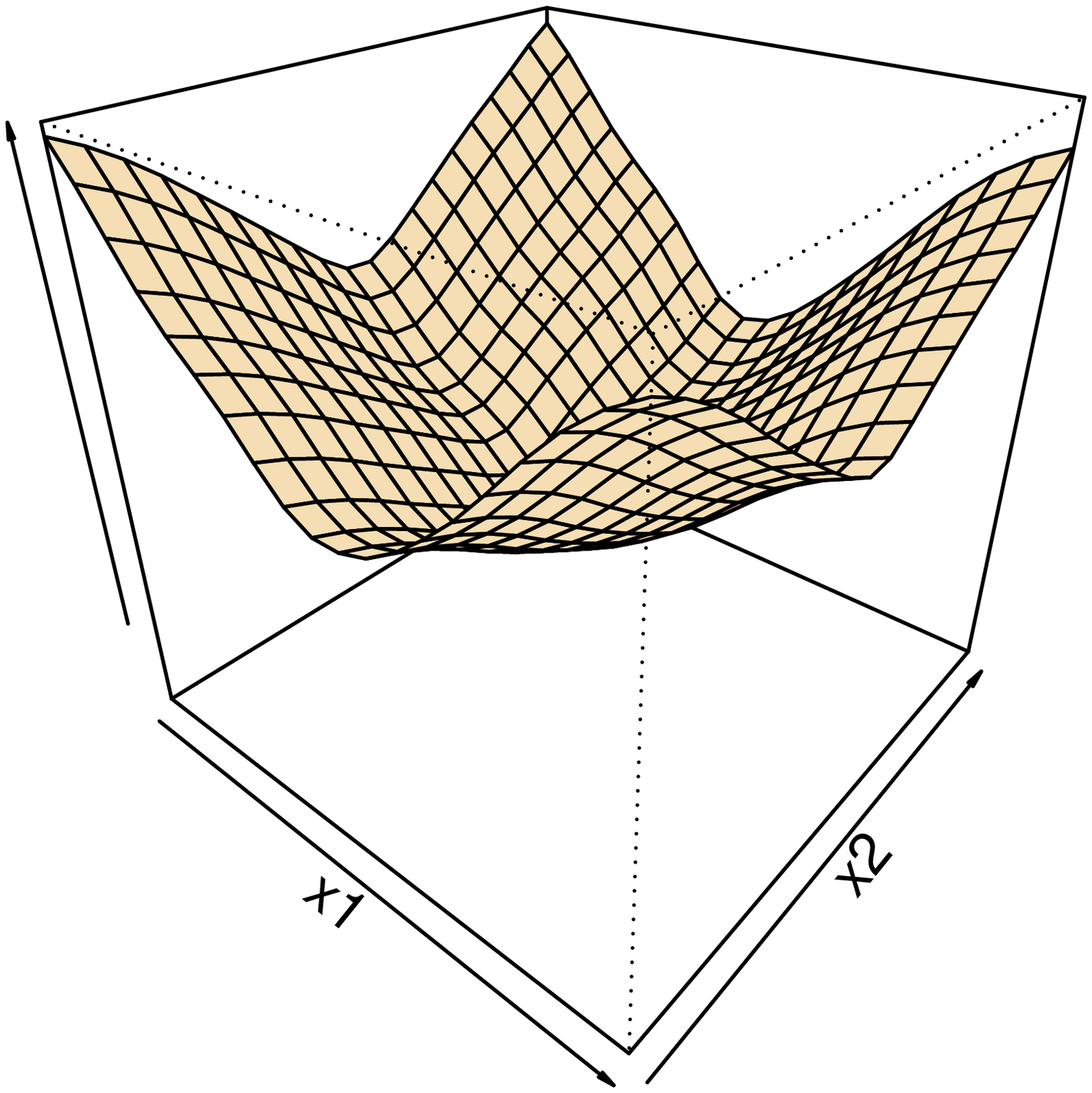}
\label{fig:subfig5}
}
\subfigure[$m_0$]{
\includegraphics[width=3.5cm]{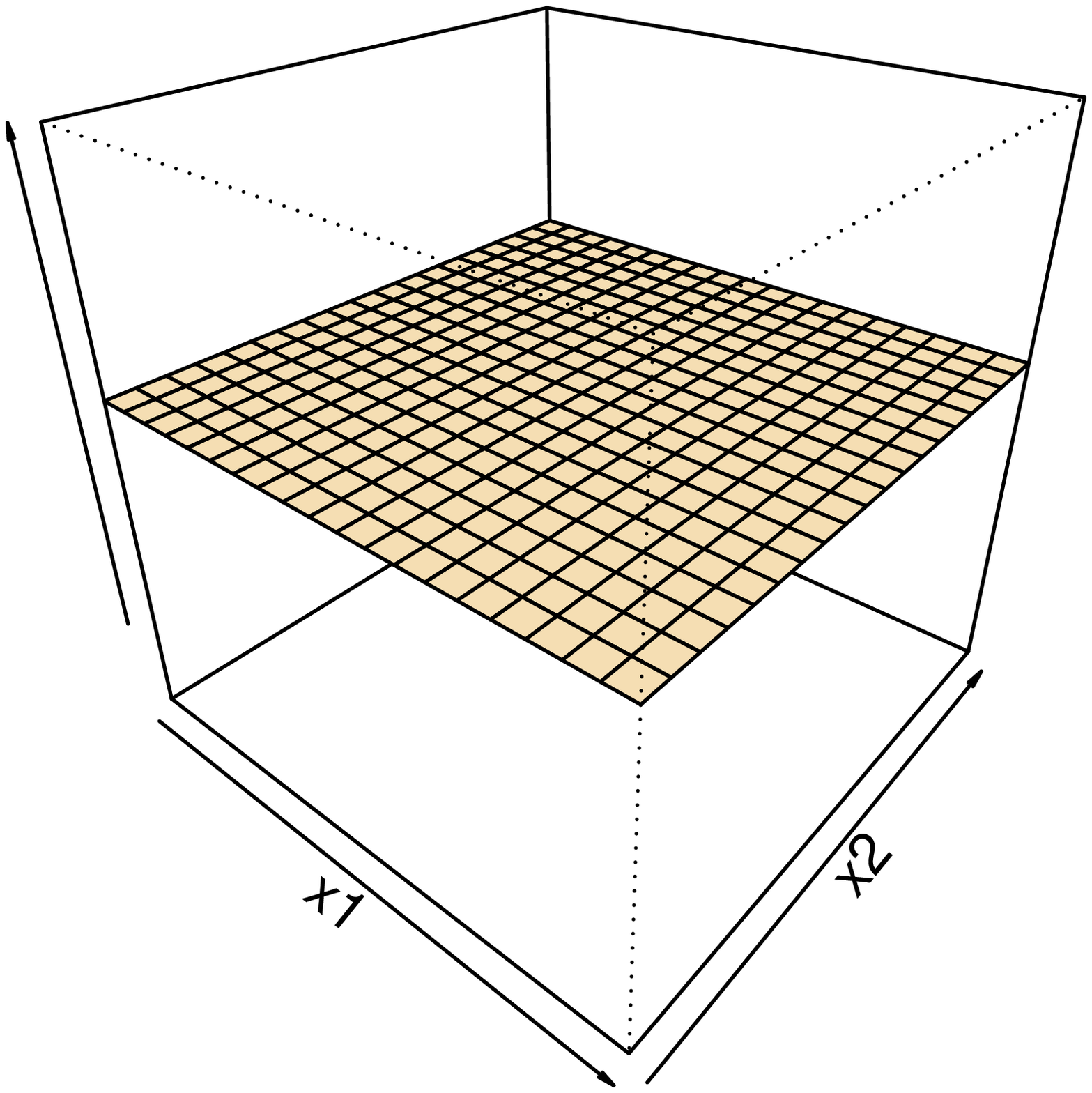}
\label{fig:subfig1}
}
\subfigure[$m_1$]{
\includegraphics[width=3.5cm]{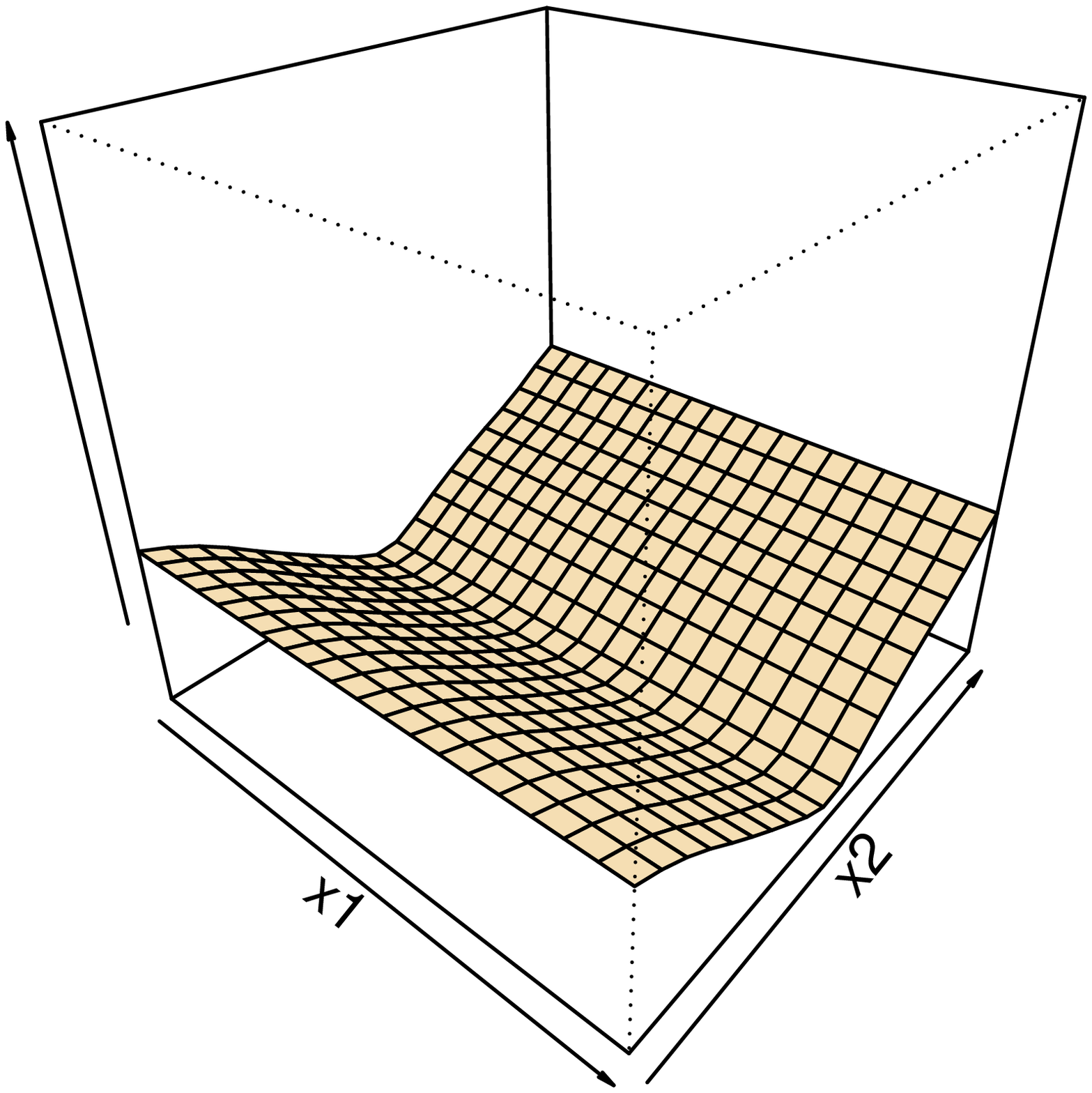}
\label{fig:subfig2}
}
\subfigure[$m_2$]{
\includegraphics[width=3.5cm]{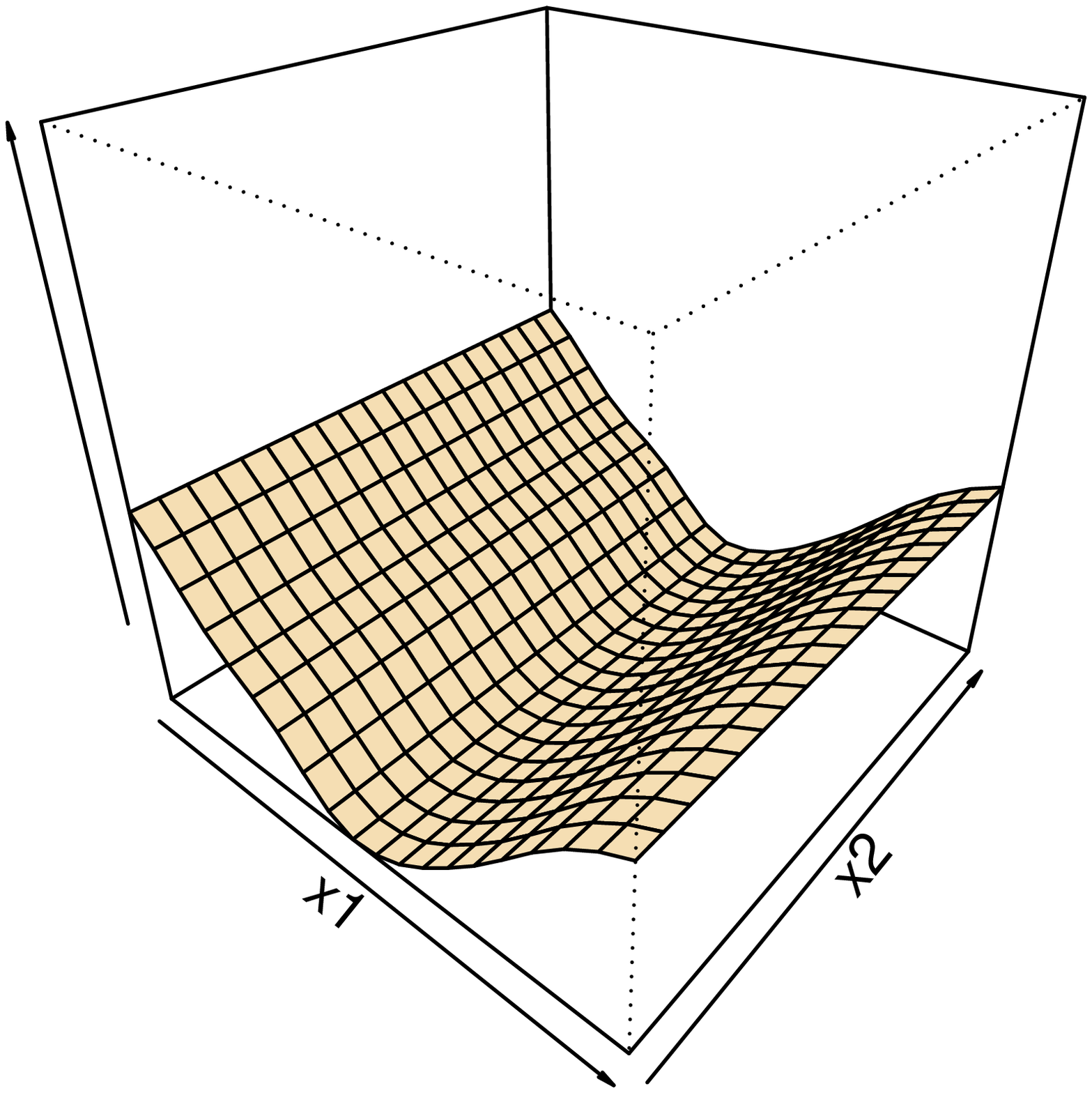}
\label{fig:subfig3}
}
\subfigure[$m_{12}$]{
\includegraphics[width=3.5cm]{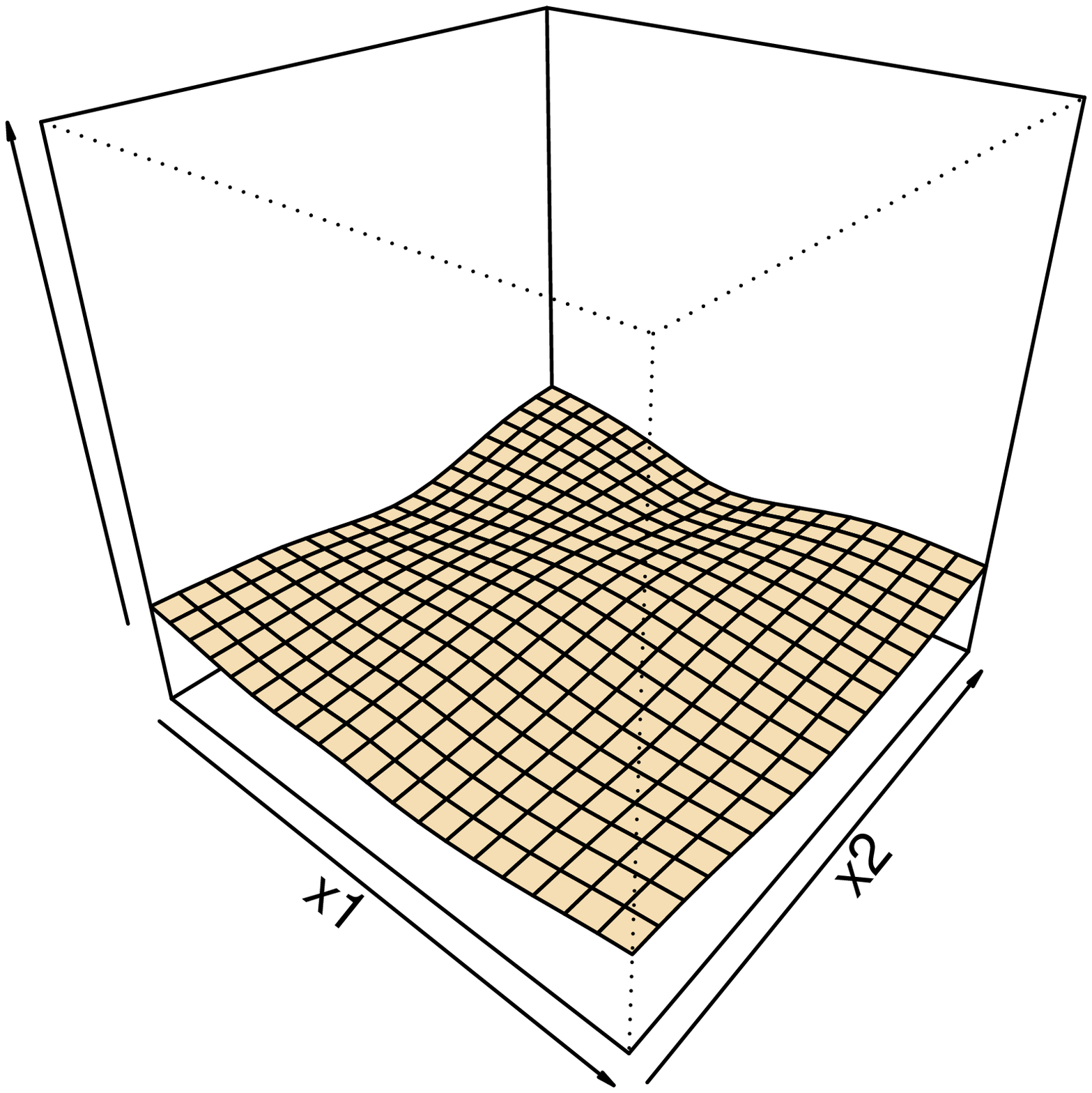}
\label{fig:subfig4}
}
\caption[No Optional caption]{Representation of the $g$-function, the model $m$ and all the submodels on $[0,1]^2$. The $z$ scale is the same on all graphs.}
\label{fig:2}
\end{figure}

\subsection{Computation of sensitivity indices}
We have seen that the terms $m_I$ of the ANOVA representation of $m$ can be obtained naturally when using kernel $K_{ANOVA}^*$. As the sensitivity indices are based on those terms, they also benefit from the structure of $K_{ANOVA}^*$ and they can be computed analytically and efficiently:

\begin{propo}
The sensitivity indices $S_I$ of $m$ are given by:
\begin{equation}
S_I= \frac{\V(m_I(\mathbf{X}_I))}{\V(m(\mathbf{X}))} = \frac{ \mathbf{F}^t \mathrm{K}^{-1} \left (\bigodot_{i \in I} \Gamma_i \right ) \mathrm{K}^{-1} \mathbf{F}}{ \mathbf{F}^t \mathrm{K}^{-1} \left ( \bigodot_{i=1}^d (1_{n \times n } + \Gamma_i) - 1_{n \times n } \right ) \mathrm{K}^{-1} \mathbf{F}}
\label{eq:Sindfinal}
\end{equation}
where $\Gamma_i$ is the $n \times n$ matrix $ \Gamma_i = \int_{D_i} \mathbf{k}^i_0(x_i) \mathbf{k}^i_0(x_i)^t \dix{x_i}{i}$ and $1_{k \times l }$ is the $k \times l$ matrix of ones.
\label{prop3}
\end{propo}
\begin{proof}
The numerator is obtained by direct calculation. Using that $\mathbf{k}^i_0$ is a vector of zero mean functions and the independence of the $X_i$ we have:
\begin{equation}
\begin{split}
\V(m_I(\mathbf{X}_I)) & = \V \left( \left ( \bigodot_{i \in I} \mathbf{k}^i_0(X_i)^t \right ) \mathrm{K}^{-1} \mathbf{F} \right) \\
& =  \mathbf{F}^t  \mathrm{K}^{-1} \C \left ( \bigodot_{i \in I} \mathbf{k}^i_0(X_i) \right ) \mathrm{K}^{-1} \mathbf{F} \\
& =  \mathbf{F}^t  \mathrm{K}^{-1}  \bigodot_{i \in I} \C \left ( \mathbf{k}^i_0(X_i) \right ) \mathrm{K}^{-1} \mathbf{F}\\
& = \mathbf{F}^t  \mathrm{K}^{-1} \bigodot_{i \in I} \left( \int_{D_i} \mathbf{k}^i_0(x_i) \mathbf{k}^i_0(x_i)^t \dix{x_i}{i} \right) \mathrm{K}^{-1} \mathbf{F} .
\end{split}
\label{eq:calcnum}
\end{equation}
For the denominator, we obtain similarly
\begin{equation}
\begin{split}
\V(m(\mathbf{X})) & =  \mathbf{F}^t  \mathrm{K}^{-1}  \bigodot_{i \in I} \left ( \int_{D_i} \left( 1_{n \times 1} + \mathbf{k}^i_0(x_i) \right) \left(1_{n \times 1} + \mathbf{k}^i_0(x_i) \right)^t \dix{x_i}{i} \right) \mathrm{K}^{-1} \mathbf{F} \\
& \qquad - \mathbf{F}^t  \mathrm{K}^{-1}  1_{n \times n} \mathrm{K}^{-1} \mathbf{F}.
\end{split}
\label{eq:calcden}
\end{equation}
We then use the property that the $k_0^i(x,.)$ are zero mean functions again to obtain $\int_{D_i} \left( 1_{n \times 1} + \mathbf{k}^i_0(x_i) \right) \left(1_{n \times 1} + \mathbf{k}^i_0(x_i) \right)^t \dix{x_i}{i} = 1_{n \times n} + \Gamma_i$.
\end{proof}

A similar property is given in~\cite{Chen2005} for usual separable kernels but the computation of $S_I$ require computation of all $S_J$ for $J \subset I$. This recursivity induces a propagation of the numerical error when the integrals in $\Gamma_i$ are approximated numerically. Here, the particular structure of $K_{ANOVA}^*$ avoids any recursivity and each index is given directly by Proposition~\ref{prop3}.

\paragraph{Example 3: Computation of sensitivity indices}
As previously, we consider the $g$-function as a test function to illustrate Proposition~\ref{prop3}. One particular asset of $g$ is that the sensitivity indices $S_I$ can be obtained analytically:
\begin{equation}
S_I = \frac{\prod_{i \in I} \frac{1}{3(1+a_i)^2}}{\prod_{k=1}^{d} \left( 1+\frac{1}{3(1+a_k)^2} \right)-1}.
\label{eq:SI}
\end{equation}
The value of the dimension is set to $d=5$ and the coefficients are arbitrarily chosen to be $(a_1,a_2,a_3,a_4,a_5) = (0.2,0.6,0.8,100,100)$. With these values, the variables $X_1$, $X_2$ and $X_3$ explain $99.99\%$ of the variance of $g(\mathbf{X})$.

\medskip

Now, we compare the indices given by Proposition~\ref{prop3} for three different kernels versus the analytical value given by Eq.~\ref{eq:SI}. The kernels we use are the Mat\'ern 3/2 kernel $k_m$ introduced in Eq.~\ref{eq:MK}, the Gaussian kernel $k_g$ of Eq.~\ref{eq:noytest} and the kernel $\tilde{k}_b(x,y)=1 + \min(x,y)$. In the latter, the value 1 is added to the kernel of the Brownian motion to relax the constraint of nullity at the origin.

\medskip

The Design of Experiment (DoE) is a 50 point maximin LHS over $[0,1]^5$. As each DoE leads to a particular model and a different value of the sensitivity indices, computation of the indices was repeated for 50 different DoEs. The obtained mean values and the standard deviations are gathered in Table~\ref{tab:1}. As for example 1, we studied for the model based on $\tilde{k}_b$ the error due to the approximation of the integrals in Eq.~\ref{eq:Sindfinal} by Riemann sums based on 100 points per dimension. The error on the final values of the indices does not appear to be significant since it is typically around $10^{-5}$.

\begin{table}[ht]
	\centering
		\begin{tabular}{|c|ccccccc|}
		\hline
		I & 1 & 2 & 3 & \{1,2\} & \{1,3\} & \{2,3\} & \{1,2,3\} \\ \hline
		Analytical   & 0.43 & 0.24 & 0.19 & 0.06 & 0.04 & 0.03 & 0.01 \\ \hline
		\multirow{2}{*}{$\tilde{k}_b$}  &  0.44  &  0.27  &  0.20  &  0.01  &  0.01  &  0.01  &  0.00 \\ 
		                              & (0.05) & (0.05) & (0.04) & (0.01) & (0.00) & (0.01) & (0.00) \\ \hline
		\multirow{2}{*}{$k_m$}        &  0.44  &  0.24  &  0.19  &  0.01  &  0.01  &  0.01  &  0.00  \\ 
		                              & (0.06) & (0.05) & (0.04) & (0.01) & (0.01) & (0.01) & (0.00) \\ \hline
		\multirow{2}{*}{$k_g$}        &  0.33  &  0.19  &  0.14  &  0.01  &  0.02  &  0.03  &  0.03  \\ 
		                              & (0.08) & (0.06) & (0.05) & (0.02) & (0.02) & (0.02) & (0.02) \\ \hline
		\end{tabular}
\caption{Mean value and standard deviation of the sensitivity indices for the $g$-function. }
\label{tab:1}		
\end{table}
The computational accuracy of the global sensitivity indices is judged to be satisfactory for $\tilde{k}_b$ and $k_m$. On the other hand, the model based on $k_g$ performs significantly worse than the others. This can be explained by the unsuitability of the Gaussian kernel for the approximation of a function that is not differentiable everywhere. This example highlights the importance of choosing a kernel suited to the problem at hand.
\medskip

Although the indices are correctly estimated with $\tilde{k}_b$ and $k_m$, the sum of the indices presented in Table~\ref{tab:1} is respectively equal to $0.94$ and $0.90$ for these kernels (whereas we obtain $1$ for the first line). It implies that the associated models present variations for some sets of indices that are not represented in this table and eventually the sparsity in the variance structure of the $g$-function is not perfectly captured by the model. As the kernels $K_{ANOVA}^*$ are ANOVA kernels, methods such as SUPANOVA~\cite{Gunn2002} or hierarchical kernel learning~\cite{Bach2009} can be applied for detecting this sparsity. The combination of $K_{ANOVA}^*$ and SUPANOVA present similarities with the COSSO~\cite{Lin2006} for splines but such developments are out of the scope of this article.

\section{Generalization to regularization problems}
\label{sec:reg}
We have considered for now the interpolation point of view and the expression of the best predictor given in Eq.~\ref{eq:krikrisol}. However, all the results can immediately be generalized to regularization problems. Indeed, the proofs of Propositions 2 and 3 are still valid if we replace the Gram matrix $\mathrm{K}$ by $\mathrm{K}+ \lambda \mathrm{I}$ so we have the following results.
\begin{coro}[Corollary of Proposition 2:]
\textit{The terms of the ANOVA representation of the regularization best predictor $\tilde{m}(\mathbf{x}) = \mathbf{k}(\mathbf{x})^t (\mathrm{K}+ \lambda \mathrm{I})^{-1} \mathbf{F}$ are}
\begin{equation}
\tilde{m}_I = \left( \bigodot_{i \in I} \mathbf{k}^i_0(x_i)^t \right) (\mathrm{K}+ \lambda \mathrm{I})^{-1} \mathbf{F}.
\label{eq:coro2}
\end{equation}
\end{coro}

\begin{coro}[Corollary of Proposition 3:]
\textit{The expression of the sensitivity indices of $\tilde{m}$ is}
\begin{equation}
S_I= \frac{ \mathbf{F}^t (\mathrm{K}+ \lambda \mathrm{I})^{-1} \left (\bigodot_{i \in I} \Gamma_i \right ) (\mathrm{K}+ \lambda \mathrm{I})^{-1} \mathbf{F}}{ \mathbf{F}^t (\mathrm{K}+ \lambda \mathrm{I})^{-1} \left ( \bigodot_{i=1}^d (1_{n \times n } + \Gamma_i) - 1_{n \times n } \right ) (\mathrm{K}+ \lambda \mathrm{I})^{-1} \mathbf{F}}.
\label{eq:coro3}
\end{equation}
\end{coro}

\paragraph{Example 4}
The aim of this last example is to illustrate two points: the corollaries we have just stated and the use of a non-uniform measure. We consider here a simple test function introduced in~\cite{saltelli2008global}:
\begin{equation}
\begin{split}
f\ : \mathds{R}^2 & \rightarrow \mathds{R} \\
\mathbf{x} & \mapsto x_1 + x_2^2 + x_1x_2.
\end{split}
\label{eq:ftest}
\end{equation}
The input space $\mathds{R} \times \mathds{R}$ is endowed with a measure $\mu_1 \otimes \mu_2$ where $\mu_i$ is the standard normal distribution. With these settings, the ANOVA representation of $f$ is
\begin{equation}
f(\mathbf{x}) = \underbrace{1}_{f_0} + \underbrace{x_1}_{f_1} + \underbrace{x_2^2 - 1}_{f_2} + \underbrace{x_1 x_2}_{f_{1,2}}
\label{eq:fANOVA}
\end{equation} 
and the sensitivity indices $S_1$, $S_2$ and $S_{1,2}$ are respectively equal to $1/4$, $1/2$, $1/4$. As we want to illustrate here a regularization issue, we assume that the observations of $f$ are affected by a noise of observation $f_{obs}=f+\varepsilon$ where $ \varepsilon $ is a specific realization of a Gaussian white noise with variance $\lambda$. 

\medskip

The DoE is a 20 point LHS maximin over $[-5,5]^2$ and we consider the following kernel over $\mathds{R}^2 \times \mathds{R}^2$:
\begin{equation}
K_{ANOVA}^*(\mathbf{x},\mathbf{y}) = 200 \: (1 + g_0(x_1,y_1)) \: (1 + g_0(x_2,y_2))
\label{eq:Kanovafin}
\end{equation}
where $g_0$ is obtained by decomposing a Gaussian kernel $k_g(x,y)=exp \left(- \left( \frac{x-y}{10} \right)^2 \right)$. The value $200$ is a scale parameter roughly estimated from the variance of the observations at the DoE. 

\medskip

Figure~\ref{fig:fin} represents the model and main effects we obtained for $\lambda=1$. It appears that despite the small number of design points, the submodels $\tilde{m}_1$ and $\tilde{m}_2$ give a good approximation of the analytical main effects $f_1$ and $f_2$.

\begin{figure}[!ht]
\centering
\subfigure[$\tilde{m}(\mathbf{x})$]{
\includegraphics[width=3.5cm]{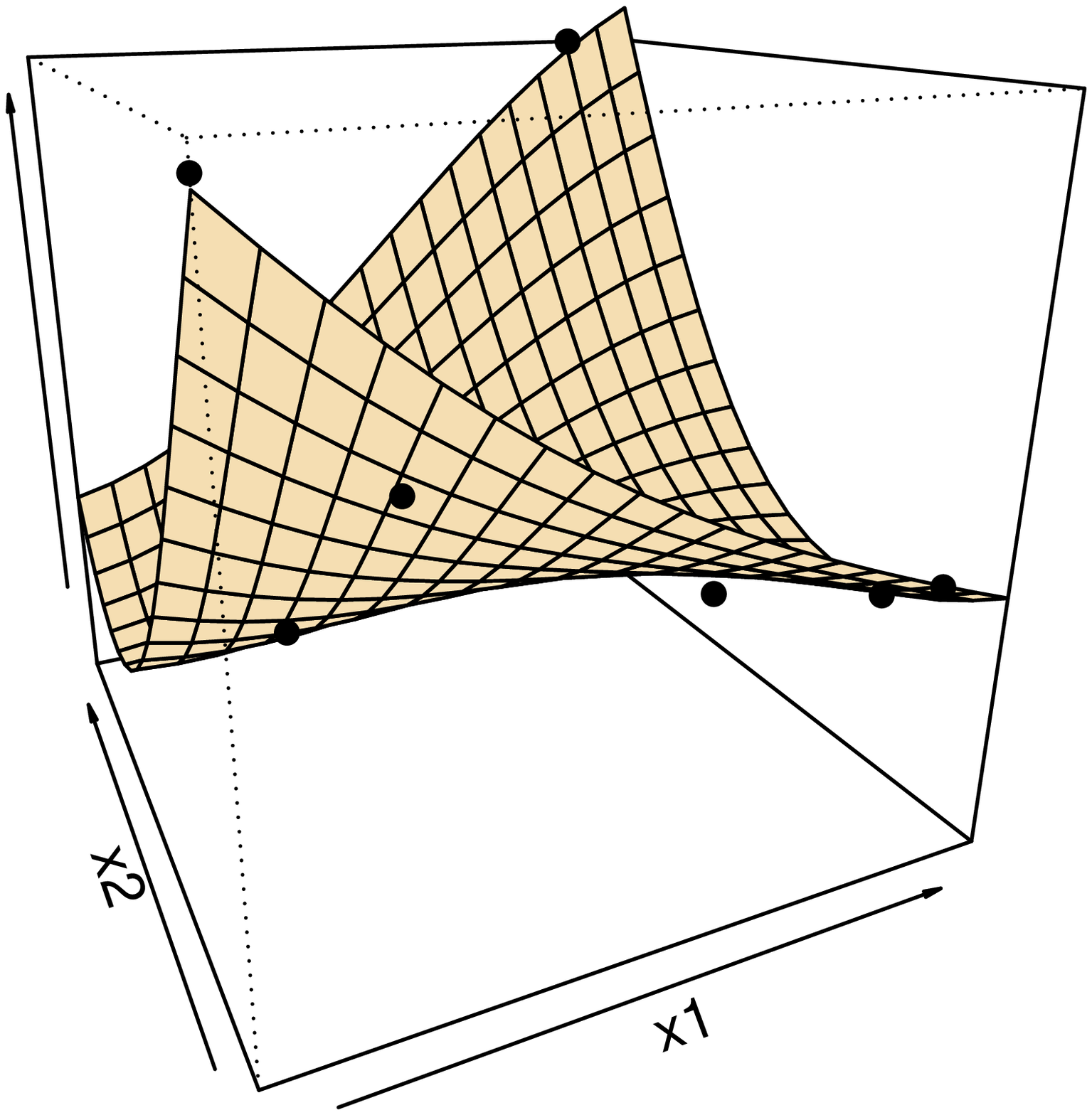}
\label{fig:subfig6a}
}
\subfigure[$\tilde{m}_1(x_1)$]{
\includegraphics[width=3.5cm]{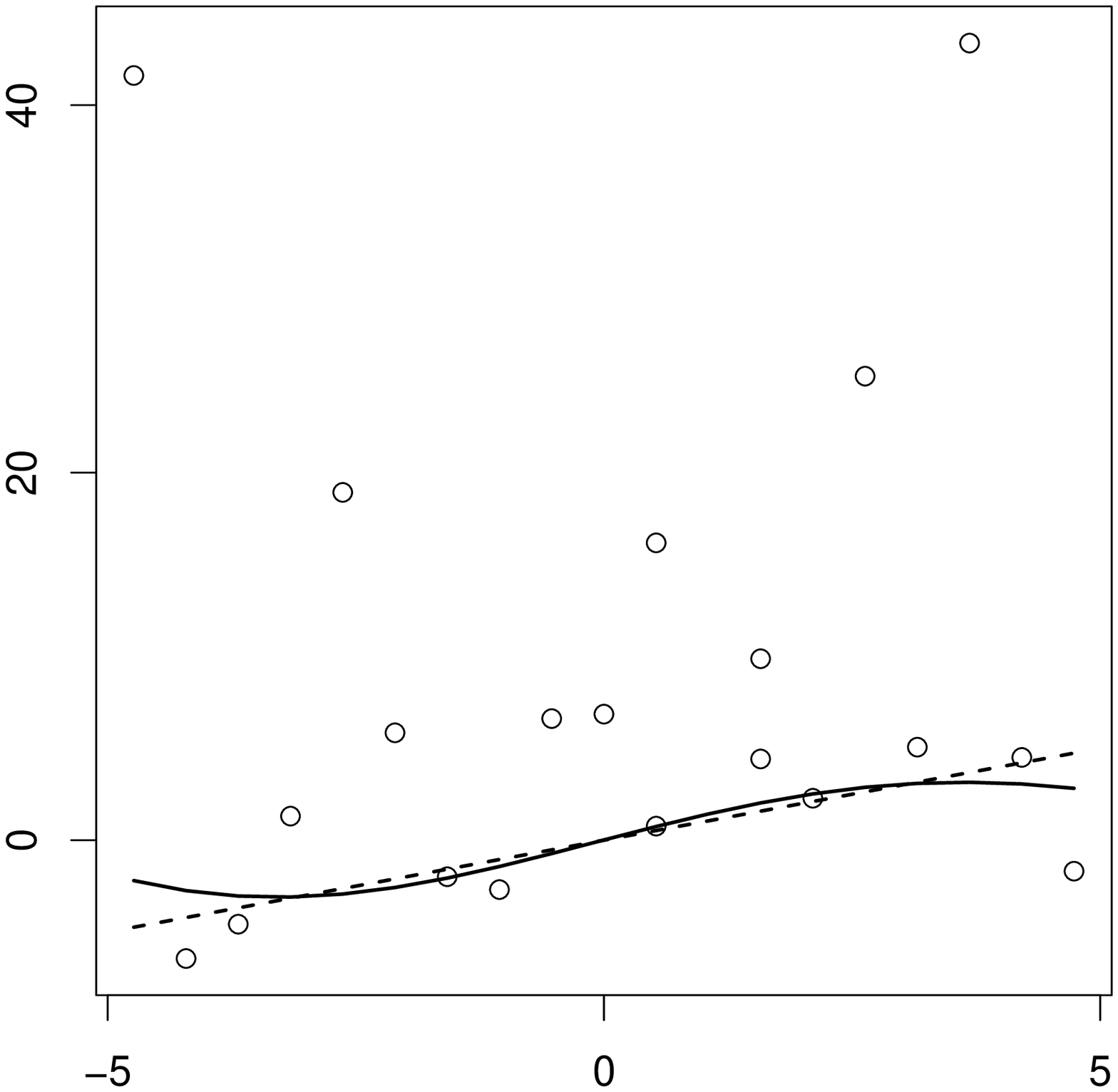}
\label{fig:subfig5b}
}
\subfigure[$\tilde{m}_2(x_2)$]{
\includegraphics[width=3.5cm]{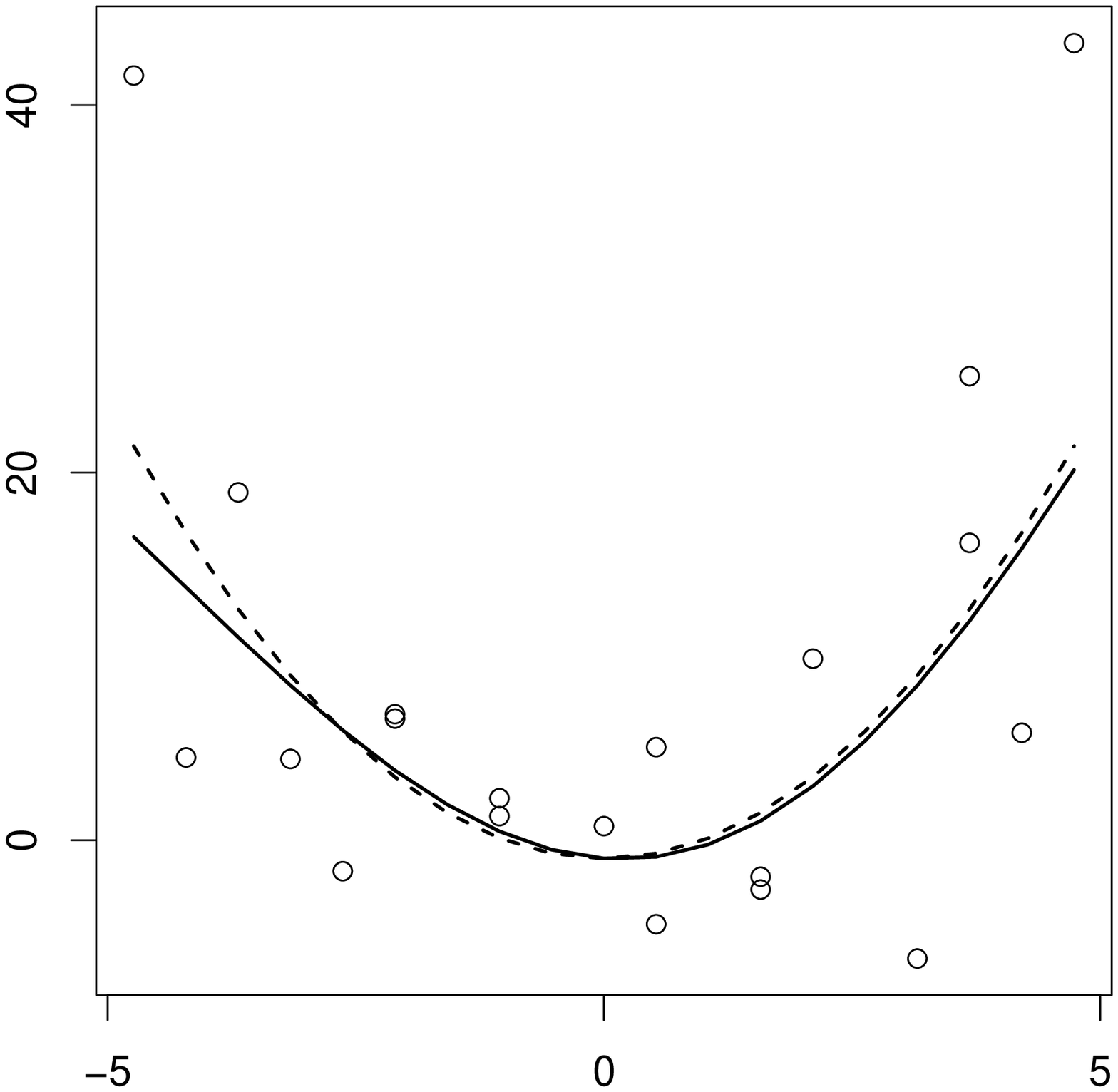}
\label{fig:subfig1c}
}
\caption[No Optional caption]{Representations of the model and two submodels for $\lambda=1$. The dots correspond to the observations and for panels (b) and (c), the dashed lines stand for the analytical main effects given in Eq.~\ref{eq:fANOVA}.}
\label{fig:fin}
\end{figure}

\medskip

Table~\ref{tab:2} gives the values of the sensitivity indices for different levels of noise. As $\V(f(\mathbf{X}))=4$ and $\V(\varepsilon(\mathbf{X}))= \lambda$, the variance of the noise we introduced is up to four times larger than the variance of the signal we want to extract. The computation of the indices appears to be very robust for this example, even for a low signal-to-noise ratio. 

\medskip

For $\lambda =0$, the model retrieves the analytical value of the parameters. When $\lambda$ increases, the standard deviation of the estimates increases and the index $S_2$ tends to be slightly underestimated. One can note that the analytical value is at the most one standard deviation away from the estimated mean. 

\begin{table}[ht]
\centering
\begin{tabular}{|c|ccc|}
		\hline
		I & 1 & 2  & \{1,2\}  \\ \hline
		Analytical   & 0.25 & 0.5 & 0.25 \\ 
		$\lambda = 0$ & 0.25 (0.00) & 0.5 (0.00) & 0.25 (0.00)\\
		$ \lambda = 1$ & 0.25 (0.05) & 0.48 (0.04) & 0.28 (0.03) \\
		$ \lambda = 2$ & 0.26 (0.06) & 0.47 (0.05) & 0.28 (0.04) \\
  	$ \lambda = 4$ & 0.24 (0.08) & 0.46 (0.06) & 0.30 (0.05) \\
		$ \lambda = 8$ & 0.26 (0.10) & 0.44 (0.08) & 0.30 (0.07) \\
		$ \lambda = 16$ & 0.28 (0.14) & 0.40 (0.11) & 0.32 (0.09) \\ \hline
\end{tabular}
\caption{Mean value and standard deviation of the sensitivity indices for the test function $f$ for various values of the noise parameter $\lambda$.}
\label{tab:2}		
\end{table}

\section{Concluding remarks}

We introduced in this article a special case of ANOVA kernels with particular univariate kernels so that the $L^2$-orthogonality to constants is respected. This new class of kernels offers good interpretation properties to the models since their functional ANOVA decomposition can be obtained analytically, without the usual recursive integral calculations. Finally, we showed that this natural decomposition of the best predictor benefits the computation of the global sensitivity indices.

\medskip

The use of Riemann sums for the numerical computation of the integrals has been discussed and the error of approximation has shown to be negligible in various examples. Up to calculation or tabulation of the integral of univariate kernels, replacing the usual ANOVA kernels with the ones proposed here may be done at negligible cost in applications, with substantial benefits for the model interpretability and global sensitivity analysis studies.

\medskip

The parameter estimation issue for $K_{ANOVA}^*$ has
not been raised yet in this article. This is however an important point
for the practical use of those kernels. The use of the likelihood theory
has been considered, but many points such as the links between the
optimal parameters for $K$ and the optimal parameters for the associated
$K_{ANOVA}^*$ need to be studied in detail.
Finally, since the pattern of the proof of Proposition~\ref{prop1} can be applied 
to any bounded operator on $\mathcal{H}$, lines of future research include 
focusing on other operators than the integral operator $I$, for example for building
RKHS respecting orthogonality to a family of trend basis functions.

\section*{Acknowledgments}
The author would like to thank the reviewers for their valuable comments. Their constructive remarks contributed significantly to the actual version of the article. We are also grateful to James Hensman and Ben Youngman for all their remarks and comments.
\bibliographystyle{plain}
\bibliography{references}

\end{document}